\documentclass{article}
\usepackage{iclr2026_conference,times}

\usepackage{amsmath,amsfonts,bm}

\def\eqref#1{equation~\ref{#1}}
\def\1{\bm{1}}

\DeclareMathAlphabet{\mathsfit}{\encodingdefault}{\sfdefault}{m}{sl}
\SetMathAlphabet{\mathsfit}{bold}{\encodingdefault}{\sfdefault}{bx}{n}

\usepackage{hyperref}
\usepackage{url}
\usepackage{graphicx}
\usepackage{booktabs}
\usepackage{amsthm}
\usepackage{amssymb}
\usepackage{mathtools}
\usepackage{microtype}
\usepackage{enumitem}
\usepackage{float}

\graphicspath{%
  {figures/}%
  {./figures/}%
  {iclr2026/figures/}%
  {./iclr2026/figures/}%
  {iclr2026/iclr2026/figures/}%
  {./iclr2026/iclr2026/figures/}%
  {../../figures/}%
  {../../code/outputs/exp1_1/results/figures/}%
}

\title{Counterfactual Identifiability Beyond Global Monotonicity: Non-Monotone Triangular Structural Causal Models}

\author{
Pengcheng Tan, Jiang Chen, Dehui Du\thanks{Corresponding author.}\\
Shanghai Key Laboratory of Trustworthy Computing, East China Normal University
}

\newtheorem{definition}{Definition}
\newtheorem{theorem}{Theorem}
\newtheorem{proposition}{Proposition}
\newtheorem{lemma}{Lemma}
\newtheorem{remark}{Remark}

\newcommand{\doop}{\ensuremath{\mathrm{do}}}
\newcommand{\EI}{\ensuremath{\mathrm{EI}}}
\newcommand{\TM}{\ensuremath{\mathrm{TM\mbox{-}SCM}}}
\newcommand{\NM}{\ensuremath{\mathrm{NM\mbox{-}TM\mbox{-}SCM}}}
\newcommand{\CC}{\ensuremath{\mathrm{CC\mbox{-}TSCM}}}

 \iclrpubliccopy
\begin{document}

\maketitle

\begin{abstract}
Structural causal models provide a unified semantics for interventions and counterfactuals, but most identifiability results still rely on global monotonicity or similarly restrictive mechanism classes. These assumptions are often violated in embodied interaction, where the same exogenous perturbation can induce opposite responses under different contact contexts. We ask what structure still suffices once global monotonicity is dropped. We introduce non-monotone triangular structural causal models (\NM), which retain triangular recursion but replace global monotonicity with mechanism-wise invertibility and context-independent inverse transport. The first ensures unique factual abduction under fixed parents; the second prevents observation-equivalent models from relabeling exogenous coordinates across contexts. Within the shared-order triangular class, we prove that these conditions are equivalent to exogenous isomorphism and therefore imply complete counterfactual identifiability, and we give a counterexample showing that local invertibility alone is insufficient. We instantiate the theory in CausalInverter, with triangular invertible layers, orientation gates, and transport-stability regularization. On synthetic mechanisms, including a continuous non-monotonicity bridge, the structural bias yields systematic counterfactual gains as non-monotonicity increases. On MuJoCo Door, our model achieves perfect event-level counterfactual recovery on the formal split, lowers continuous counterfactual angle error relative to a compact Transformer baseline, and in a three-seed head-to-head delivers substantially more stable event recovery than Transformer and conditional-flow predictors. On MuJoCo Push, where non-monotonicity is much weaker, the same low-data state predictors remain competitive or better, consistent with a bias-variance boundary. These results identify a broader identifiable regime between globally monotone triangular models and unconstrained black-box world models.
\end{abstract}

\section{Introduction}

Counterfactual reasoning requires comparing alternative outcomes under the \emph{same factual background}. In embodied systems, this matters for explanation, diagnosis, and safe planning: would a trajectory have changed under a different action, contact angle, or control magnitude?

Structural causal models (SCMs) provide the standard semantics for this type of question \citep{pearl2009causality, shpitser2008complete, peters2017elements}. However, SCM semantics alone do not make counterfactuals identifiable from observational data. The difficulty is one of model-class identifiability: if multiple SCMs induce the same observational distribution, must they also agree on interventions and counterfactuals? Classical positive results answer this question by imposing strong asymmetries on the mechanism class, such as additive noise, post-nonlinear structure, location-scale restrictions, or monotone/invertible triangular parametrizations \citep{hoyer2009nonlinear, zhang2010distinguishing, peters2014causal, immer2023lsnm, khemakhem2021causal, javaloy2023causalnf, chen2025exogenous}.

Many physical systems violate the regularity underlying these results. Contact, friction, and geometry can induce context-dependent sign reversals: the same exogenous perturbation may push the system in opposite directions under different parent states. In such settings, globally monotone model classes are too restrictive, but removing structure entirely restores counterfactual ambiguity. Our question is therefore simple: \emph{if global monotonicity is abandoned, what structure still preserves counterfactual identifiability?}

Our answer is that global monotonicity is not the essential ingredient. What matters is that each local mechanism remains invertible under fixed parents, and that the inverse transport between observation-equivalent models does not vary with the parent context. The first property guarantees unique factual abduction; the second prevents latent relabeling across contexts. Figure~\ref{fig:intro-overview} summarizes this shift from fixed orientation to invariance-based structure.

\begin{figure}[t]
\centering
\includegraphics[width=\linewidth]{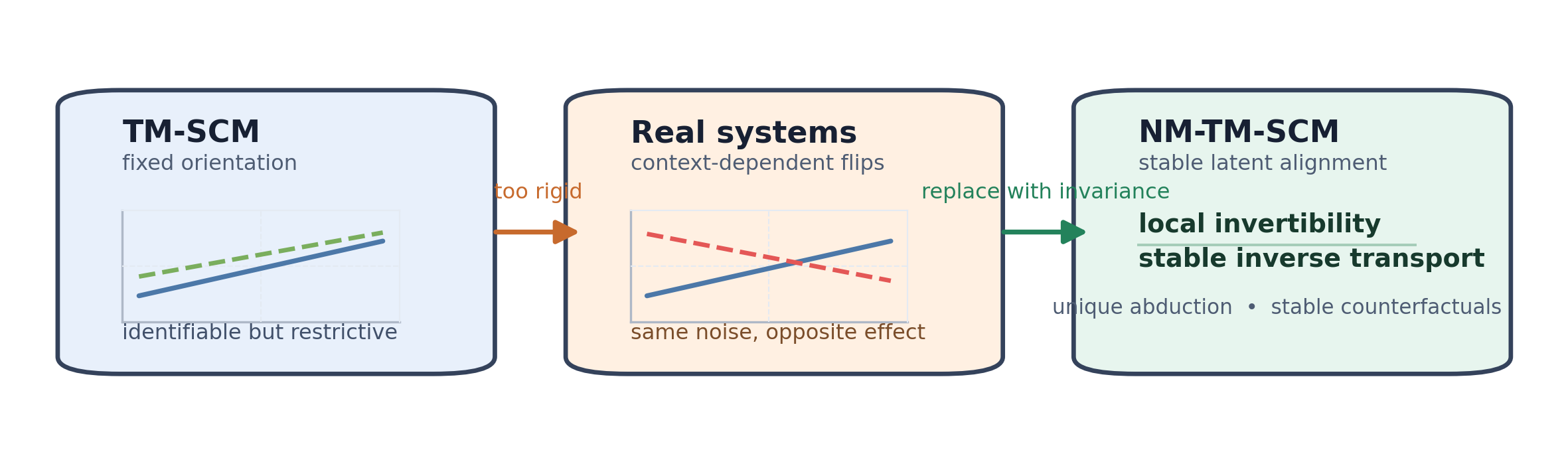}
\caption{From fixed orientation to stable exogenous alignment.}
\label{fig:intro-overview}
\end{figure}

Based on this perspective, we introduce \NM, a strict superset of \TM. The class permits context-dependent orientation flips while preserving stable exogenous alignment. Our theoretical result can be summarized by the chain
\[
\begin{aligned}
\text{mechanism-wise invertibility} \\
\quad + \text{context-independent inverse transport}
&\Longleftrightarrow \text{exogenous isomorphism} \\
&\Longrightarrow \text{complete counterfactual identifiability}.
\end{aligned}
\]

This viewpoint also motivates a structured causal inverter that models local inversion, context-dependent orientation, and stable counterfactual transport.

The main contributions are as follows:
\begin{itemize}[leftmargin=1.2em]
\item \textbf{Theory contribution:} We propose \NM, a non-monotone triangular SCM class, and show that mechanism-wise invertibility together with inverse-transport invariance is equivalent to exogenous isomorphism, which in turn guarantees complete counterfactual identifiability. We also give a constructive counterexample showing that local invertibility alone is insufficient, thereby clarifying the sharpness of the assumption set.
\item \textbf{Method contribution:} We instantiate the theory in CausalInverter, a learnable triangular inverter with context-dependent orientation gates and transport-stability regularization, so that the theoretical conditions become an explicit model design rather than a purely abstract guarantee.
\item \textbf{Experimental contribution:} We validate the method on synthetic mechanisms and MuJoCo interaction tasks. The Door benchmark provides the main physical evidence, while Push serves as a boundary case that makes the method's operating regime explicit and exposes the tradeoff between causal identifiability constraints and task complexity.
\end{itemize}

\section{Setup and Model Class}

\subsection{Triangular recursive SCMs}

Let
\[
M=\langle \mathcal{I}, \Omega_V, \Omega_U, f, P_U\rangle
\]
be a recursive SCM with endogenous variables $V=(V_1,\dots,V_d)$ and exogenous variables $U=(U_1,\dots,U_d)$. We assume a fixed causal order and write the structural equations as
\begin{align}
V_1 &= f_1(U_1), \nonumber\\
V_2 &= f_2(V_1,U_2), \nonumber\\
&\vdots \nonumber\\
V_d &= f_d(V_{1:d-1},U_d). \label{eq:triangular}
\end{align}
This defines a solution map $\Gamma_M : \Omega_U \to \Omega_V$, and the observational distribution is
\[
P_V = (\Gamma_M)_{\sharp} P_U.
\]

We focus on intervention and counterfactual queries of the form
\[
P\!\left(V_T^{\doop(V_A=a)} \in B \mid E=e\right).
\]

\subsection{Exogenous isomorphism}

The recent counterfactual identifiability literature has highlighted \emph{exogenous isomorphism} as the right notion of equivalence for counterfactual semantics \citep{nasr2023bijective, chen2025exogenous, chen2024exomatching, zhou2024domaincf}. Intuitively, two observation-equivalent models should count as counterfactually identical only if they share the same latent worlds up to a coordinate change in exogenous space.

\begin{definition}[Exogenous isomorphism]
Two shared-order triangular SCMs $M$ and $M'$ are exogenously isomorphic, written $M \sim_{\EI} M'$, if there exists a coordinate-wise bijection
\[
\psi = (\psi_1,\dots,\psi_d) : \Omega_U \to \Omega_U'
\]
such that
\[
P_{U'} = \psi_{\sharp} P_U
\]
and, for every $i$ and every parent context $v_{<i}$,
\[
f_i'(v_{<i},\psi_i(u_i)) = f_i(v_{<i},u_i).
\]
\end{definition}

\paragraph{Physical meaning in embodied systems.}
In a physical interaction task, the exogenous coordinates represent latent but fixed properties of a single rollout: small pose perturbations, actuator lag, frictional micro-variation, or hidden contact geometry. Exogenous isomorphism therefore has a concrete interpretation. It says that two observationally equivalent models may use different internal coordinates, but they must still refer to the \emph{same physical latent episode} one-to-one. The same hidden shove, contact offset, or friction realization cannot be relabeled as a different latent world merely because the parent context changes. This is exactly the requirement needed for counterfactual reasoning in robotics: when we ask what would have happened under a different action, we want to hold the latent physical background fixed rather than silently swap it for another one.

\subsection{From context-controlled triangular models to \NM}

We first separate mechanism-level invertibility from global monotonicity.

\begin{definition}[Context-controlled triangular SCM]
A triangular recursive SCM belongs to \CC\ if, for every $i$ and every fixed parent context $v_{<i}$, the section
\[
u_i \mapsto f_i(v_{<i},u_i)
\]
is a bijection. The orientation of this bijection is allowed to depend on $v_{<i}$.
\end{definition}

This class sits strictly between globally monotone triangular models and unconstrained triangular SCMs. For example,
\[
X = U_X, \qquad Y = \operatorname{sgn}(X) U_Y
\]
belongs to \CC\ but not to \TM, because the sign of the mechanism flips with the parent context.

\begin{definition}[\NM]
A shared-order triangular SCM belongs to \NM\ if it satisfies:
\begin{enumerate}[leftmargin=1.4em]
\item \textbf{Triangular recursion:} the SCM has a unique solution map $\Gamma_M$ under a fixed causal order.
\item \textbf{Mechanism-wise invertibility:} for every $i$ and every fixed $v_{<i}$, the map $u_i \mapsto f_i(v_{<i},u_i)$ is a bijection.
\item \textbf{Context-independent inverse transport:} for every observation-equivalent model $M'$ with the same order, there exists a bijection $\psi_i$ such that
\[
(f_i'(v_{<i},\cdot))^{-1} \circ f_i(v_{<i},\cdot) = \psi_i
\]
for all $v_{<i}$.
\end{enumerate}
\end{definition}

\paragraph{Why segmented \TM\ is not a replacement.}
It may seem that a non-monotone process could always be decomposed into monotone phases and then handled by \TM\ piecewise. This does not solve the same problem. A segmented \TM\ formulation requires additional assumptions that are typically stronger than ours: the phase boundary must itself be identifiable from the available causal context; the post-intervention phase-transition rule must be modelled and identifiable; and the exogenous semantics must remain aligned across phases. Otherwise, non-monotonicity is merely shifted into the phase switcher, and the counterfactual path remains ambiguous. In this sense, segmented \TM\ is a special case inside our framework when the switching structure is itself known or identifiable, not a replacement for the problem studied here.

\paragraph{Concrete counterexample to identifiable segmentation.}
Consider a latent stick-slip variable $S=\mathbf{1}[U_S \le \sigma(X)]$, where $X=U_X$ is observed, $U_S \sim \mathrm{Unif}(0,1)$ is unobserved, and $\sigma(\cdot)\in(0,1)$ is a context-dependent contact law. Let
\[
Y=(2S-1)U_Y,
\]
with $U_Y$ supported on positive values. Conditional on a fixed $S$, the mechanism is monotone in $U_Y$, so a segmented \TM\ view is possible only if the phase label is known. But here the phase boundary is itself hidden inside $U_S$ and changes under intervention through $\sigma(X)$. After observing $(X,Y)$, one does not recover a context-only segmentation rule; after intervening on $X$, one must also know how the latent phase would have switched. In embodied terms, this is the familiar case where sticking versus slipping depends on unobserved contact micro-state rather than on the visible pose alone. A segmented \TM\ model can represent this process only by assuming oracle access to the hidden phase or to its post-intervention transition rule, which is precisely the extra burden that our formulation avoids.

\section{Theory}

Our theoretical program has three steps: mechanism-wise invertibility induces a global exogenous code, context-independent inverse transport is equivalent to exogenous isomorphism, and exogenous isomorphism yields complete counterfactual identifiability. Relative to recent results on bijective SCMs and exogenous alignment \citep{nasr2023bijective, chen2025exogenous}, our goal is to isolate exactly which part of global monotonicity is structurally necessary once context-dependent orientation flips are allowed.

\begin{lemma}[Global bijection from local invertibility]
\label{lem:global-bijection}
If a triangular recursive SCM satisfies mechanism-wise invertibility, then its solution map $\Gamma_M : \Omega_U \to \Omega_V$ is a bijection.
\end{lemma}

\noindent\textbf{Proof sketch.}
Injectivity follows by induction along the causal order: if $\Gamma_M(u)=\Gamma_M(\tilde u)$, then $u_1=\tilde u_1$, and the same argument applies recursively because each local section is bijective under the same parent context. Surjectivity follows by recursive inversion of the observed coordinates. \qed

\begin{theorem}[Invariant inverse transport iff exogenous isomorphism]
\label{thm:transport-ei}
Let $M$ and $M'$ be shared-order triangular SCMs with
\[
P_V^M = P_V^{M'}.
\]
If both satisfy mechanism-wise invertibility, then the following are equivalent:
\begin{enumerate}[leftmargin=1.4em]
\item for every mechanism $i$, the cross-model inverse transport
\[
(f_i'(v_{<i},\cdot))^{-1} \circ f_i(v_{<i},\cdot)
\]
is independent of $v_{<i}$;
\item $M \sim_{\EI} M'$.
\end{enumerate}
\end{theorem}

\noindent\textbf{Proof sketch.}
By Lemma~\ref{lem:global-bijection}, both solution maps are invertible. If the inverse transports are context-independent, they define coordinate-wise bijections $\psi_i$. Composing these along the triangular order yields
\[
\Gamma_{M'} \circ \psi = \Gamma_M,
\]
which implies both mechanism-level compatibility and $P_{U'}=\psi_{\sharp}P_U$. The converse follows immediately by left-composing the exogenous-isomorphism identity with $(f_i'(v_{<i},\cdot))^{-1}$. \qed

\begin{remark}[Why inverse transport invariance is essentially sharp]
\label{rem:sharpness}
Theorem~\ref{thm:transport-ei} should be read as a sharp separation result inside the shared-order triangular class. Once mechanism-wise invertibility is fixed, the remaining obstruction to exogenous isomorphism is exactly the context dependence of cross-model inverse transport. If that dependence disappears, exogenous alignment follows automatically. If it is allowed, Proposition~\ref{prop:counterexample} shows that observation-equivalent models can already disagree on counterfactuals. In this precise sense, inverse transport invariance is not merely a convenient sufficient assumption: it is the smallest \emph{context-free mechanism-level} condition in our analysis that rules out context-dependent relabeling of latent physical worlds.
\end{remark}

\begin{theorem}[Exogenous isomorphism implies complete counterfactual identifiability]
\label{thm:ei-cf}
If $M \sim_{\EI} M'$, then for any intervention set $A$, intervention value $a$, evidence event $E=e$, target set $T$, and measurable set $B$,
\[
P_M\!\left(V_T^{\doop(V_A=a)} \in B \mid E=e\right)
=
P_{M'}\!\left(V_T^{\doop(V_A=a)} \in B \mid E=e\right).
\]
\end{theorem}

\noindent\textbf{Proof sketch.}
The exogenous isomorphism $\psi$ establishes a pointwise correspondence between latent worlds. Under the same intervention, the triangular recursion produces the same potential response at matched exogenous coordinates. Since $P_{U'}=\psi_{\sharp}P_U$ and the factual evidence also satisfies $\Gamma_{M'} \circ \psi = \Gamma_M$, the entire abduction--action--prediction pipeline is preserved. \qed

\begin{proposition}[Local invertibility alone is insufficient]
\label{prop:counterexample}
There exist observation-equivalent triangular SCMs with mechanism-wise invertibility but context-dependent inverse transport that disagree on counterfactuals.
\end{proposition}

A minimal example is
\[
M: \quad X=U_X,\qquad Y=\operatorname{sgn}(X)U_Y
\]
versus
\[
M': \quad X=U_X',\qquad Y=U_Y',
\]
with independent standard Gaussian exogenous variables. They induce the same observational distribution, but for the factual sample $(X=1,Y=y)$, the counterfactual under $X \leftarrow -1$ is $-y$ in $M$ and $y$ in $M'$. The failure comes exactly from the fact that
\[
(f'(x,\cdot))^{-1}\circ f(x,\cdot)=\operatorname{sgn}(x)\cdot(\cdot)
\]
depends on the context.

Together, Theorems~\ref{thm:transport-ei} and~\ref{thm:ei-cf} show that \NM\ preserves the full counterfactual semantics usually attributed to stronger globally monotone classes, while covering mechanisms that \TM\ cannot represent.

\section{CausalInverter}

The theory suggests that learning should target three properties simultaneously: accurate observational modelling, stable inversion of factual samples into exogenous coordinates, and stable transport of those coordinates under counterfactual rollout. We encode these ideas in CausalInverter.

\subsection{Structured parametrization}

For each mechanism, we use the parametrization
\[
v_i = f_{i,\theta}(v_{<i},u_i)
= m_{i,\theta}(v_{<i}) +
s_{i,\theta}(v_{<i}) \cdot q_{i,\theta}(u_i;v_{<i}),
\]
where $m_{i,\theta}$ is a context-dependent shift, $q_{i,\theta}(\cdot;v_{<i})$ is an invertible monotone scalar flow, and $s_{i,\theta}(v_{<i}) \in \{-1,+1\}$ is a context-dependent orientation gate. If $s_{i,\theta}$ is constant across contexts, the model collapses to a monotone triangular form; if it changes with context, the model can express parent-dependent sign reversals.

\subsection{Training objective}

Given the global solution map $\Gamma_{\theta}$ and its inverse, we optimize the observational negative log-likelihood
\[
\mathcal{L}_{\mathrm{nll}}
=
-\frac{1}{N}\sum_{n=1}^N
\log p_U(\Gamma_\theta^{-1}(v^{(n)}))
-\frac{1}{N}\sum_{n=1}^N
\log \left|\det J_{\Gamma_\theta^{-1}}(v^{(n)})\right|,
\]
augmented with two stabilizers:
\begin{align}
\mathcal{L}_{\mathrm{cyc}}
&=
\mathbb{E}_{v \sim P_{\mathrm{data}}}
\|\Gamma_\theta(\Gamma_\theta^{-1}(v)) - v\|_2^2 \nonumber\\
&\quad+
\mathbb{E}_{u \sim p_U}
\|\Gamma_\theta^{-1}(\Gamma_\theta(u)) - u\|_2^2, \\
\mathcal{L}_{\mathrm{tr}}
&=
\sum_{i=1}^d
\mathbb{E}_{c,\tilde c,z}
\Big(\|\partial_z K_{i,\theta}(z;c,\tilde c)\|_2 - \bar\kappa_i(c,\tilde c)\Big)^2,
\end{align}
where $K_{i,\theta}$ is the learned context transport between two parent contexts. We also regularize the relaxed gate field $\tilde s_{i,\theta}$ to avoid high-frequency orientation jitter. The full objective is
\[
\mathcal{L}
=
\mathcal{L}_{\mathrm{nll}}
+ \lambda_{\mathrm{cyc}} \mathcal{L}_{\mathrm{cyc}}
+ \lambda_{\mathrm{tr}} \mathcal{L}_{\mathrm{tr}}
+ \lambda_{\mathrm{ori}} \mathcal{L}_{\mathrm{ori}}.
\]

\subsection{Counterfactual inference}

For a factual sample $v$, we first recover the exogenous code
\[
u^\star = \Gamma_\theta^{-1}(v),
\]
then execute the intervention in the structural equations and roll out the post-intervention response recursively while keeping $u^\star$ fixed. This directly mirrors the abduction--action--prediction semantics, but without requiring a globally fixed monotone orientation.

\section{Experiments}

Our experiments ask whether the proposed structural constraints improve \emph{individual-level} counterfactual recovery. We use synthetic mechanisms to isolate the effect and MuJoCo interaction tasks to test whether it survives in physical systems, reporting both geometric rollout errors and event-level counterfactual metrics \citep{cadei2024smoke}.

\subsection{Synthetic mechanisms}

We evaluate three mechanism families: globally monotone, threshold-flip, and smooth-flip. The benchmark covers $108$ configurations in total: three mechanism families, four exogenous noise families, three dimensions $d \in \{3,5,10\}$, and three sample sizes $N \in \{2\times 10^3,10^4,5\times 10^4\}$. We report counterfactual MSE, latent recovery error, and direction accuracy.

Table~\ref{tab:synthetic-main} shows the core pattern. When the mechanism is globally monotone, all structured baselines are competitive. Once parent-dependent orientation flips appear, stronger generative flexibility alone does not solve the problem: ANM, monotone triangular transport, and contextual bijective baselines all remain around the same counterfactual error level, whereas CausalInverter sharply reduces the error while recovering the orientation field with accuracy around $0.97$. This improvement is not only visible in the means. Using paired Wilcoxon signed-rank tests over the matched synthetic configurations within each mechanism family, we find that on both threshold-flip and smooth-flip mechanisms, our CF-MSE is significantly lower than both TM-SCM Quantile and BSCM-Flow Contextual (all $p<10^{-10}$), with bootstrap confidence intervals for the mean difference staying strictly below zero; see Appendix Table~\ref{tab:appendix-synth-significance}.

We also run a continuous bridge sweep that varies non-monotonicity strength rather than mechanism identity. Its calibrated synthetic score rises from $0.00$ to $0.98$, while the mean CF-MSE gain of our model over TM-SCM grows from $0.0008$ to $0.8409$; across runs, this gain tracks the score with slope $1.017$ and Spearman $\rho=0.821$ ($p=2.8\times 10^{-40}$). This directly supports the later bias-variance reading: monotone transport remains competitive when non-monotonicity is weak, but inverse-transport alignment becomes increasingly necessary once branch-sensitive reversals become common; see Appendix Figure~\ref{fig:appendix-synth-bridge} and Table~\ref{tab:appendix-synth-bridge}.

\begin{table}[t]
\caption{Synthetic benchmark summary.}
\label{tab:synthetic-main}
\centering
\small
\begin{tabular}{lccccc}
\toprule
Mechanism & ANM & TM-SCM & BSCM-Flow & Ours & Ours Dir. Acc. \\
\midrule
Global monotone & 0.0071 & 0.0017 & 0.0036 & \textbf{0.0000} & 1.0000 \\
Threshold flip & 0.7288 & 0.7286 & 0.7297 & \textbf{0.1611} & 0.9761 \\
Smooth flip & 0.5934 & 0.5929 & 0.5941 & \textbf{0.0961} & 0.9745 \\
\bottomrule
\end{tabular}
\end{table}

\begin{figure}[t]
\centering
\includegraphics[width=\linewidth]{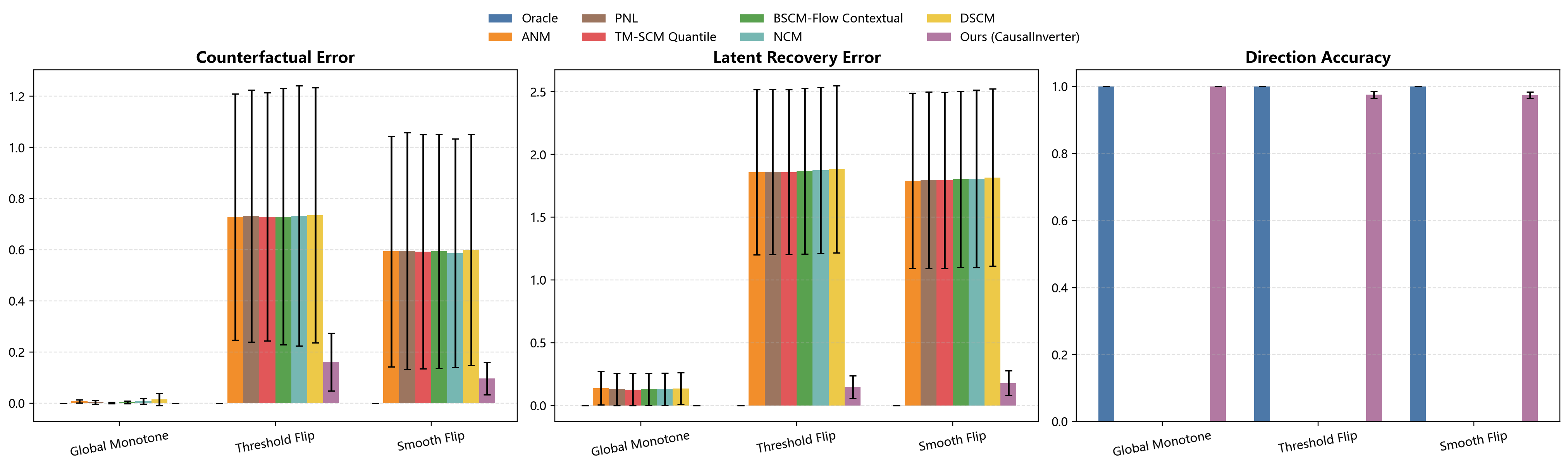}
\caption{Synthetic main panel.}
\label{fig:synthetic-main}
\end{figure}

\subsection{Physical interaction tasks}

We evaluate two low-dimensional MuJoCo tasks with contact-dependent dynamics: Push and Door. Both use the same formal \texttt{balanced\_pool\_v2} counterfactual sampler. For each factual trajectory, we first generate a candidate pool of interventions and then select at most one query while balancing success-change and no-change cases, and covering the transition types $S\!\to\!S$, $S\!\to\!F$, $F\!\to\!S$, and $F\!\to\!F$ as much as the task permits.

\paragraph{Baseline scope.}
The physical comparison is intentionally state-based. We compare against aligned non-oracle baselines that use the same observed state, the same rollout horizon, and the same counterfactual protocol: a linear dynamics model, a mode-conditioned context model, an MLP world model, an autoregressive dynamics model, a GRU sequence model, a compact Transformer sequence model, and a conditional-flow dynamics model, plus an oracle environment upper bound. The main table emphasizes representative stronger predictors and keeps the context model in the appendix-level diagnostics and ablations for space. This is a controlled structural comparison against stronger low-data state predictors, not a full benchmark against image-conditioned world models \citep{hafner2021dreamerv2, zhang2024pivotr}. We also do not include return-conditioned policy sequence models, since the evaluation target here is state-conditioned counterfactual rollout rather than action generation.

\paragraph{A composite non-monotonicity score.}
To avoid over-interpreting a single flip-rate number, we summarize each task by
\[
\mathsf{NMS} \triangleq \rho_{\mathrm{flip}}\cdot \rho_{\mathrm{contact}},
\]
where $\rho_{\mathrm{flip}}$ is the empirical Context Flip Ratio and $\rho_{\mathrm{contact}}$ is the contact ratio. This exposition metric captures both how often the local orientation changes and how concentrated those changes are in the interaction phase. Using coarse bins, we regard $\mathsf{NMS}<0.25$ as weak, $0.25\le\mathsf{NMS}<0.5$ as moderate, and $\mathsf{NMS}\ge 0.5$ as strong. Under the formal test splits, Push has $\mathsf{NMS}\approx 0.175$ and Door has $\mathsf{NMS}\approx 0.696$, so Door remains the main physical evidence while Push functions as a boundary case.

\begin{table}[t]
\caption{Main MuJoCo results under the balanced counterfactual protocol. We report representative low-data state baselines spanning linear, feedforward, autoregressive, recurrent, Transformer, and flow families; the oracle environment upper bound is perfect in both tasks. `CF Success Acc.` denotes accuracy on the counterfactual success/failure label, not the raw fraction of successful openings.}
\label{tab:physical-main}
\centering
\scriptsize
\begin{tabular}{llccc}
\toprule
Task & Model & CF MSE & CF Success Acc. & CF Success Change Acc. \\
\midrule
Push & Linear Dynamics & \textbf{0.0017} & \textbf{0.7500} & \textbf{0.7500} \\
Push & MLP World Model & 0.0071 & \textbf{0.7500} & \textbf{0.7500} \\
Push & Autoreg. Dyn. & 0.0029 & 0.5312 & 0.5312 \\
Push & GRU Seq. Model & 0.0029 & 0.5000 & 0.5000 \\
Push & Transformer Seq. & 0.0021 & 0.6562 & 0.6562 \\
Push & Cond. Flow Dyn. & 0.0030 & 0.6562 & 0.6562 \\
Push & Ours & 0.0027 & 0.5625 & 0.5625 \\
\midrule
Door & Linear Dynamics & 0.2276 & 0.8750 & 0.8750 \\
Door & MLP World Model & 0.1369 & \textbf{1.0000} & \textbf{1.0000} \\
Door & Autoreg. Dyn. & 0.2219 & 0.9062 & 0.9062 \\
Door & GRU Seq. Model & 0.5911 & 0.3125 & 0.3125 \\
Door & Transformer Seq. & 0.1918 & \textbf{1.0000} & \textbf{1.0000} \\
Door & Cond. Flow Dyn. & \textbf{0.0596} & 0.8438 & 0.8438 \\
Door & Ours & 0.1082 & \textbf{1.0000} & \textbf{1.0000} \\
\bottomrule
\end{tabular}
\end{table}

Table~\ref{tab:physical-main} sharpens the boundary-aware empirical picture. On Push, once stronger state-based baselines are included, our method is no longer dominant: Linear Dynamics and the MLP world model both reach CF success accuracy $0.75$, the conditional-flow and Transformer baselines reach $0.6562$, and the compact Transformer attains the lowest CF box-MSE at $0.0021$. Our model remains competitive on continuous geometry, but does not win either the event metric or the scalar counterfactual error. On Door, the stronger evidence is branch-sensitive rather than uniformly metric-dominant. On the formal split, CausalInverter, the MLP world model, and the compact Transformer all attain perfect event accuracy, but our model has the lowest continuous counterfactual angle error among those perfect-event models ($0.1082$ versus $0.1369$ for MLP and $0.1918$ for Transformer). The conditional-flow baseline attains the lowest scalar CF Door-MSE overall ($0.0596$) yet drops to event accuracy $0.8438$, so its local geometric fit does not reliably preserve the correct counterfactual branch. Paired tests reinforce this reading: against Transformer, our CF Door-MSE is significantly lower (mean difference $-0.0836$, Wilcoxon $p=3.39\times 10^{-2}$), while against conditional flow our event accuracy is significantly higher (McNemar $p=3.12\times 10^{-2}$); see Appendix Table~\ref{tab:appendix-physical-significance}. The event metric itself is also interpretable: the formal Door split contains only $10/32$ true-success counterfactuals and $22/32$ true-failure ones, so a model can score $1.0$ by staying on the correct side of the threshold even with noticeable continuous error. Appendix Figure~\ref{fig:appendix-door-threshold} shows that Transformer succeeds with only a $0.0048$ minimum positive margin on success queries, whereas CausalInverter keeps a larger $0.0300$ buffer. The substantive claim is therefore narrower but stronger: in the strong-\(\mathsf{NMS}\) Door regime, the structured inverse most reliably preserves the event-level branch while maintaining competitive continuous accuracy.

\begin{figure}[t]
\centering
\includegraphics[width=\linewidth]{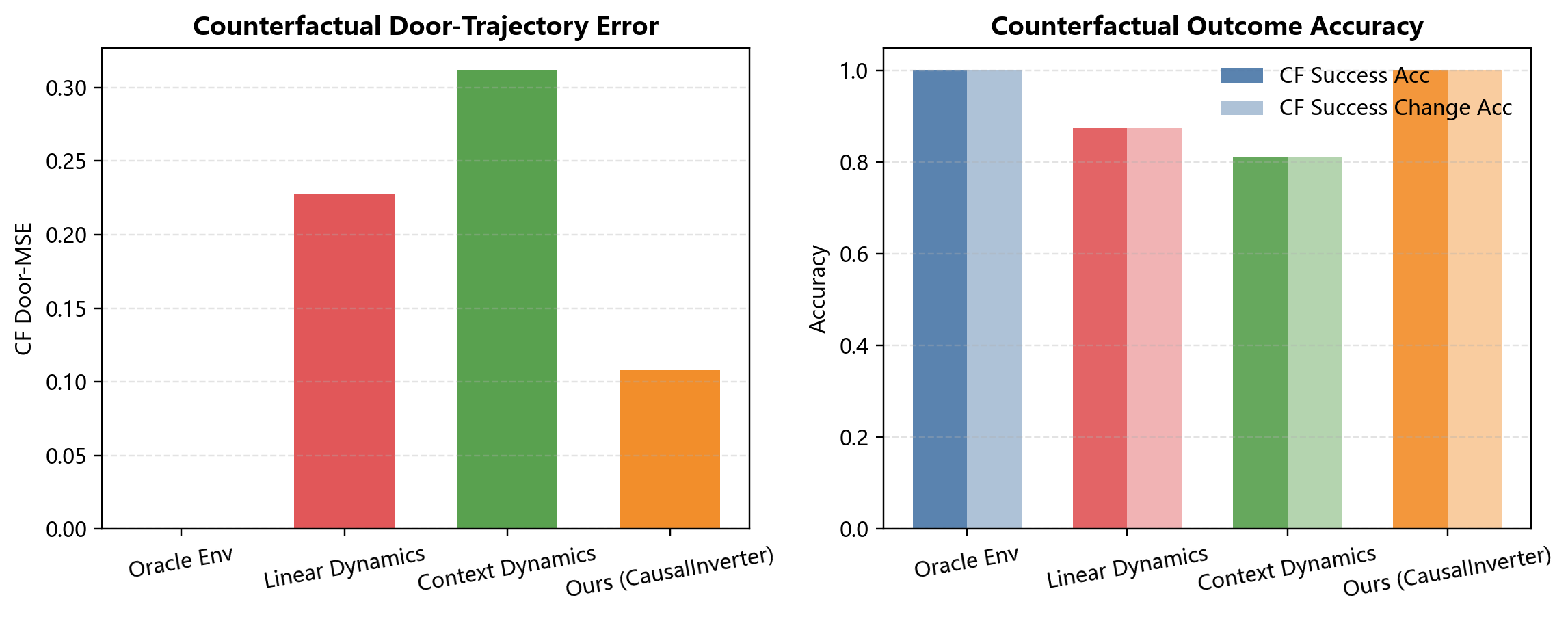}
\caption{Door main results.}
\label{fig:door-main}
\end{figure}

\paragraph{Why Push remains a boundary case.}
The theoretical explanation is best understood as an inductive-bias tradeoff. Adding inverse-transport invariance shrinks the hypothesis class, which reduces variance when the task has enough non-monotone structure to make the counterfactual path genuinely ambiguous. This is the Door regime. Push is different: its composite non-monotonicity score is much lower, so the same constraint can become over-regularizing relative to stronger low-data predictors. In this weak-\(\mathsf{NMS}\) regime, the event label is highly sensitive to small geometric errors near the success boundary, and the theory-inspired bias is not strong enough to dominate flexible state predictors. Push therefore does not refute the theory; it marks the boundary where the structural bias ceases to be the best empirical tradeoff.

\paragraph{Formal multi-seed Push evaluation.}
To verify that the Push conclusion is not a single-split accident, we rerun the formal balanced protocol with seeds $\{7,17,27\}$, producing $96$ counterfactual queries. This robustness replication follows the previously stress-tested linear/context comparison rather than re-training every stronger neural baseline across all seeds. Our method attains CF box-MSE $0.0028 \pm 0.0002$, much lower than Context Dynamics ($0.0121 \pm 0.0011$) and close to Linear Dynamics ($0.0020 \pm 0.0003$), but the event-level success accuracy remains $0.4688 \pm 0.1552$. This multi-seed replication is consistent with the stronger single-split baseline comparison above: Push does not provide stable event-level evidence for the structural bias.

\paragraph{Door robustness.}
We further evaluate Door in two complementary ways. First, a direct three-seed head-to-head among the strongest learned baselines shows that CausalInverter is the most stable event-level model: its CF success accuracy is $0.9583 \pm 0.0295$, versus $0.7083 \pm 0.1031$ for conditional flow and $0.5521 \pm 0.3231$ for Transformer, while the corresponding CF Door-MSE values are $0.1087 \pm 0.0090$, $0.0500 \pm 0.0139$, and $0.3740 \pm 0.1288$; see Appendix Tables~\ref{tab:appendix-door-head2head} and \ref{tab:appendix-door-head2head-seedwise}. We do not extend this head-to-head table to every baseline family: the MLP already appears in the formal single-split main table and the paired significance table, whereas the multi-seed comparison is intentionally kept compact around the three strongest non-oracle models that represent distinct counterfactual rollout families (sequence, flow, and structured inverse). Second, under the broader base, observation-noise, and dynamics-noise robustness suite, our method attains CF success accuracy $0.9583 \pm 0.0295$, $0.9062 \pm 0.1112$, and $0.8021 \pm 0.0531$. The main physical support therefore remains robust at the event level, even though the best continuous geometry error is not always achieved by our model.

\paragraph{Ablation ranking and explanation diagnostics.}
The structured components are easiest to separate on Door, where the task is strongly non-monotone. Removing the gate, transport, or inverse-stability component increases CF Door-MSE from $0.1082$ to $0.6260$, $0.2908$, and $0.2667$, respectively, showing that context-dependent gating is the dominant component. Push shows a different pattern: the full model and the no-gate variant are nearly tied in CF box-MSE ($0.0027$ versus $0.0027$), while removing transport or inverse stability degrades the error to $0.0055$ and $0.0038$. In the lower-\(\mathsf{NMS}\) regime, module contributions are only partially separated and thresholded event metrics remain unstable, so Door carries the main module-ordering claim.

Beyond CF-MSE and success accuracy, we also report \emph{counterfactual explanation-quality diagnostics}: endpoint, stage-wise, and mode-conditioned errors, and on Door additionally contact accuracy, context-sign accuracy, and mode-distribution mismatch. These diagnostics become deliberately mixed once stronger baselines are included. On Door, our model improves over the classical linear/context baselines on endpoint and event recovery, but the stronger neural baselines can equal or surpass it on several continuous and contact-alignment proxies; see Appendix Table~\ref{tab:appendix-door-diagnostics}. We therefore treat these quantities as secondary diagnostics rather than as the main conclusion. The primary empirical message is more specific: under strong non-monotonicity, our structured inverse remains a parameter-efficient route to event-level branch recovery, not a uniform optimizer of every continuous proxy.
Figure~\ref{fig:door-main} gives the corresponding visual summary for the Door task.

The combined evidence supports a clear empirical boundary: the proposed structural bias is most useful when the task contains strong branch-like non-monotone contact structure and the evaluation depends on selecting the correct counterfactual branch, as in Door. In that regime it yields a compact, parameter-efficient solution with the strongest event-level recovery and the most stable branch preservation across seeds. On weakly non-monotone tasks such as Push, stronger low-data state predictors can remain competitive or better.

\section{Related Work}

Classical SCM work characterizes when interventions and counterfactuals are identifiable from observed distributions, often by imposing asymmetric functional restrictions such as additive-noise, post-nonlinear, or location-scale structure \citep{pearl2009causality, shpitser2008complete, peters2017elements, hoyer2009nonlinear, zhang2010distinguishing, peters2014causal, immer2023lsnm, dong2024partiallinear}. Our question is different: within shared-order triangular models, when must observationally equivalent SCMs agree on counterfactual semantics?

Recent work on bijective SCMs, exogenous isomorphism, and invertible causal transports shows that observational equivalence alone is insufficient unless models preserve the same exogenous worlds up to a compatible coordinate change \citep{nasr2023bijective, chen2025exogenous, xia2023neural, balgi2024counterfactually, pawlowski2020deep, chen2024exomatching, yan2023counterfactualgen, zhou2024domaincf, rahman2024conditionalgen, tan2024partialid, galhotra2024conditioning, khemakhem2021causal, wehenkel2021graphical, javaloy2023causalnf, manela2024marginalflows}. World-model work is complementary: predictive latent dynamics can support planning, but predictive quality alone need not imply causal validity \citep{hafner2020dream, hafner2021dreamerv2, zhang2024pivotr, cadei2024smoke}. We extend this line by isolating the extra invariance needed once global monotonicity is dropped. A compact family-level comparison is deferred to Appendix Table~\ref{tab:related-comparison}.

\section{Conclusion}

We studied counterfactual identifiability beyond global monotonicity in shared-order triangular SCMs. Mechanism-wise invertibility together with context-independent inverse transport is sufficient for complete counterfactual identifiability within this class because it is equivalent to exogenous isomorphism. This yields the broader identifiable family \NM, which accommodates context-dependent orientation reversals excluded by \TM. Empirically, the synthetic benchmarks, especially the continuous bridge sweep, and MuJoCo Door provide the clearest support for this bias. Push remains the intended boundary case: once stronger low-data predictors are included, the structural bias does not dominate there. On Door, the evidence is sharper: our structured inverse does not optimize every continuous proxy, but it achieves the most reliable event-level branch recovery and the strongest multi-seed stability among the strongest compact learned baselines we tested.

\bibliographystyle{iclr2026_conference}
\bibliography{iclr2026_conference}

\clearpage
\appendix

\section{Full Proofs}

\subsection{Proof of Lemma~\ref{lem:global-bijection}}

\noindent\textbf{Injectivity.}
Assume $\Gamma_M(u)=\Gamma_M(\tilde u)$. We show by induction on the causal order that $u_i=\tilde u_i$ for all $i$.

For $i=1$, the equality $\Gamma_M(u)=\Gamma_M(\tilde u)$ implies
\[
f_1(u_1)=f_1(\tilde u_1).
\]
Since $u_1 \mapsto f_1(u_1)$ is bijective, we obtain $u_1=\tilde u_1$.

Now assume $u_j=\tilde u_j$ for all $j<i$. Then the corresponding endogenous prefixes coincide:
\[
V_{1:i-1}(u)=V_{1:i-1}(\tilde u).
\]
Because $\Gamma_M(u)=\Gamma_M(\tilde u)$, we also have
\[
f_i(V_{1:i-1}(u),u_i)=f_i(V_{1:i-1}(\tilde u),\tilde u_i).
\]
The parent prefixes are equal, so this reduces to
\[
f_i(v_{<i},u_i)=f_i(v_{<i},\tilde u_i)
\]
for the common context $v_{<i}=V_{1:i-1}(u)$. Mechanism-wise invertibility implies $u_i=\tilde u_i$.

Thus $u=\tilde u$, and $\Gamma_M$ is injective.

\medskip
\noindent\textbf{Surjectivity.}
Let $v=(v_1,\dots,v_d)\in\Omega_V$ be arbitrary. We construct $u=(u_1,\dots,u_d)$ recursively such that $\Gamma_M(u)=v$.

First define
\[
u_1 = f_1^{-1}(v_1),
\]
which is well-defined because $f_1$ is bijective. Suppose $u_1,\dots,u_{i-1}$ have already been defined so that the first $i-1$ coordinates generated by the SCM equal $v_{1:i-1}$. Then define
\[
u_i = f_i(v_{<i},\cdot)^{-1}(v_i).
\]
Again this is well-defined by mechanism-wise invertibility under the fixed context $v_{<i}$.

Proceeding recursively up to $i=d$ yields an exogenous vector $u$ such that $\Gamma_M(u)=v$. Hence $\Gamma_M$ is surjective.

Therefore $\Gamma_M$ is a bijection. \qed

\subsection{Proof of Theorem~\ref{thm:transport-ei}}

We prove the two directions separately.

\subsubsection*{(1) Context-independent inverse transport implies exogenous isomorphism}

Assume that for every $i$ there exists a bijection $\psi_i$ such that
\[
(f_i'(v_{<i},\cdot))^{-1}\circ f_i(v_{<i},\cdot)=\psi_i
\]
for all parent contexts $v_{<i}$. Equivalently,
\[
f_i'(v_{<i},\psi_i(u_i))=f_i(v_{<i},u_i)
\]
for every $u_i$ and every $v_{<i}$.

Define $\psi=(\psi_1,\dots,\psi_d)$. We show
\[
\Gamma_{M'}(\psi(u))=\Gamma_M(u)
\]
for every $u\in\Omega_U$ by induction on the causal order.

For $i=1$,
\[
V_1^{M'}(\psi_1(u_1)) = f_1'(\psi_1(u_1)) = f_1(u_1)=V_1^M(u_1).
\]
Assume that the first $i-1$ endogenous coordinates agree:
\[
V_{1:i-1}^{M'}(\psi(u))=V_{1:i-1}^M(u).
\]
Then
\begin{align*}
V_i^{M'}(\psi(u))
&=
f_i'(V_{1:i-1}^{M'}(\psi(u)),\psi_i(u_i)) \\
&=
f_i'(V_{1:i-1}^{M}(u),\psi_i(u_i)) \\
&=
f_i(V_{1:i-1}^{M}(u),u_i) \\
&=
V_i^M(u).
\end{align*}
So all endogenous coordinates match, and therefore
\[
\Gamma_{M'}\circ\psi=\Gamma_M.
\]

Now use observational equivalence:
\[
(\Gamma_M)_{\sharp}P_U = P_V^M = P_V^{M'} = (\Gamma_{M'})_{\sharp}P_{U'}.
\]
Since $\Gamma_{M'}\circ\psi=\Gamma_M$, we obtain
\[
(\Gamma_{M'})_{\sharp}(\psi_{\sharp}P_U)= (\Gamma_M)_{\sharp}P_U = (\Gamma_{M'})_{\sharp}P_{U'}.
\]
By Lemma~\ref{lem:global-bijection}, $\Gamma_{M'}$ is bijective, so its pushforward is injective on probability measures. Hence
\[
P_{U'}=\psi_{\sharp}P_U.
\]
Together with the mechanism identity already shown, this proves $M\sim_{\EI}M'$.

\subsubsection*{(2) Exogenous isomorphism implies context-independent inverse transport}

Now assume $M\sim_{\EI}M'$. Then for every $i$ and every $v_{<i}$,
\[
f_i'(v_{<i},\psi_i(u_i))=f_i(v_{<i},u_i).
\]
Left-compose with $(f_i'(v_{<i},\cdot))^{-1}$ to obtain
\[
\psi_i = (f_i'(v_{<i},\cdot))^{-1}\circ f_i(v_{<i},\cdot).
\]
Since the left-hand side does not depend on $v_{<i}$, neither does the right-hand side. Therefore the inverse transport is context-independent.

This completes the proof. \qed

\subsection{Proof of Theorem~\ref{thm:ei-cf}}

Assume $M\sim_{\EI}M'$ through $\psi=(\psi_1,\dots,\psi_d)$.

\subsubsection*{Step 1: pointwise equality of interventional responses}

Fix an intervention $\doop(V_A=a)$. For $u\in\Omega_U$, define $u'=\psi(u)$. We show by induction on the causal order that
\[
V^{M,\doop(V_A=a)}(u)=V^{M',\doop(V_A=a)}(u').
\]

If $i\in A$, both models assign the same intervened value $a_i$.

If $i\notin A$, assume the first $i-1$ post-intervention coordinates are equal. Then
\begin{align*}
V_i^{M',\doop(V_A=a)}(u')
&=
f_i'\!\left(V_{<i}^{M',\doop(V_A=a)}(u'),u_i'\right) \\
&=
f_i'\!\left(V_{<i}^{M,\doop(V_A=a)}(u),\psi_i(u_i)\right) \\
&=
f_i\!\left(V_{<i}^{M,\doop(V_A=a)}(u),u_i\right) \\
&=
V_i^{M,\doop(V_A=a)}(u).
\end{align*}
So the interventional potential responses agree pointwise after exogenous matching.

\subsubsection*{Step 2: equality of interventional distributions}

For any measurable target set $B$,
\begin{align*}
P_M\!\left(V_T^{\doop(V_A=a)}\in B\right)
&=
\int \mathbf{1}\!\left[V_T^{M,\doop(V_A=a)}(u)\in B\right]\,dP_U(u) \\
&=
\int \mathbf{1}\!\left[V_T^{M',\doop(V_A=a)}(\psi(u))\in B\right]\,dP_U(u).
\end{align*}
Using $P_{U'}=\psi_{\sharp}P_U$ gives
\[
P_M\!\left(V_T^{\doop(V_A=a)}\in B\right)
=
P_{M'}\!\left(V_T^{\doop(V_A=a)}\in B\right).
\]

\subsubsection*{Step 3: equality of factual evidence and posterior conditioning}

Let $E=e$ be factual evidence. Define the corresponding measurable subsets of exogenous space:
\[
\mathcal{U}_e^M = \{u:\Gamma_M(u)\in e\}, \qquad
\mathcal{U}_e^{M'} = \{u':\Gamma_{M'}(u')\in e\}.
\]
Because $\Gamma_{M'}\circ\psi=\Gamma_M$, we have
\[
u\in \mathcal{U}_e^M \quad \Longleftrightarrow \quad \psi(u)\in \mathcal{U}_e^{M'}.
\]
Hence the factual posterior sets correspond exactly under $\psi$, and the induced posteriors are pushforward-equivalent. Applying Step 2 under these conditional measures yields
\[
P_M\!\left(V_T^{\doop(V_A=a)}\in B \mid E=e\right)
=
P_{M'}\!\left(V_T^{\doop(V_A=a)}\in B \mid E=e\right).
\]

Therefore exogenous isomorphism implies complete counterfactual identifiability. \qed

\subsection{Proof of Proposition~\ref{prop:counterexample}}

Consider
\[
M:\quad X=U_X,\qquad Y=\operatorname{sgn}(X)U_Y,
\]
and
\[
M':\quad X=U_X',\qquad Y=U_Y',
\]
where $U_X,U_Y,U_X',U_Y'$ are independent standard Gaussians.

\subsubsection*{Observational equivalence}

In both models, $X$ is standard Gaussian. Conditional on $X=x$, model $M$ gives
\[
Y \mid X=x = \operatorname{sgn}(x)U_Y.
\]
Since $U_Y$ is symmetric around zero, $\operatorname{sgn}(x)U_Y$ has the same distribution as $U_Y$. Therefore
\[
P_M(Y\mid X=x)=P_{M'}(Y\mid X=x)
\]
for all $x$, and the two models induce the same joint observational distribution.

\subsubsection*{Counterfactual disagreement}

Now fix the factual sample $(X=1,Y=y)$.

In model $M$, factual abduction yields
\[
U_X=1,\qquad U_Y=y,
\]
because $\operatorname{sgn}(1)=1$. Under the intervention $X\leftarrow -1$, the same exogenous state gives
\[
Y_{X\leftarrow -1} = \operatorname{sgn}(-1)\,U_Y = -y.
\]

In model $M'$, factual abduction yields
\[
U_X'=1,\qquad U_Y'=y,
\]
and under the same intervention,
\[
Y_{X\leftarrow -1}=U_Y'=y.
\]
So the two models disagree on the counterfactual outcome despite being observationally equivalent.

\subsubsection*{Source of failure}

The cross-model inverse transport is
\[
(f'(x,\cdot))^{-1}\circ f(x,\cdot)=\operatorname{sgn}(x)\cdot(\cdot),
\]
which depends explicitly on $x$. Thus mechanism-wise invertibility alone is not sufficient: what fails is precisely context-independent inverse transport. \qed

\subsection{Why segmented \TM\ needs stronger assumptions}

We make the discussion in Section~3 explicit with a hidden-phase example. Let
\[
X=U_X,\qquad S=\mathbf{1}[U_S\le \sigma(X)],\qquad Y=(2S-1)U_Y,
\]
where $U_S\sim\mathrm{Unif}(0,1)$ is unobserved, $U_Y$ is supported on positive values, and $\sigma:\mathbb{R}\to(0,1)$ is a context-dependent contact law.

\subsubsection*{Why a segmented monotone description exists}

If the phase label $S$ were given, then the mechanism for $Y$ would be monotone inside each segment:
\[
Y=U_Y \quad \text{when } S=1,\qquad
Y=-U_Y \quad \text{when } S=0.
\]
Thus a segmented \TM\ representation is possible only in an augmented system that reveals the phase.

\subsubsection*{Why this does not solve the observed-variable problem}

In the observed system, the phase boundary is not a deterministic function of the visible context $X$ alone. It is mediated by the hidden micro-state $U_S$. Moreover, under an intervention $X\leftarrow x'$, the counterfactual phase becomes
\[
S^{\doop(X=x')}=\mathbf{1}[U_S\le \sigma(x')],
\]
so the post-intervention phase transition depends on the same hidden variable. Therefore a segmented \TM\ model over observed variables is not self-contained: to answer the counterfactual, it must additionally know either
\begin{enumerate}[leftmargin=1.4em]
\item the latent phase $S$ itself, or
\item the hidden transition rule carried by $U_S$.
\end{enumerate}
These are exactly the extra assumptions referred to in the main text. In embodied terms, this is the common situation where sticking, slipping, or latch engagement depends on contact micro-geometry that is not recoverable from the visible pose alone.

\subsection{Further interpretation of exogenous isomorphism}

The role of exogenous isomorphism can also be stated operationally. In an embodied task, a counterfactual query is meant to compare two action choices under the same latent rollout-specific disturbances. Exogenous isomorphism is the condition that preserves those disturbances one-to-one across models. Without it, two models may agree on all visible trajectories while disagreeing on which hidden contact realization, friction realization, or pose perturbation produced a given factual sample. Counterfactual disagreement then follows not from different physics in the visible state space, but from a silent mismatch in what counts as ``the same latent episode.'' This is why exogenous isomorphism is the right bridge between the mathematical theory and the physical semantics of counterfactual reasoning.

\section{Balanced Counterfactual Sampling Protocol}

\subsection{The \texttt{balanced\_pool\_v2} procedure}

Both formal MuJoCo tasks use the same counterfactual sampling logic. The procedure is:

\begin{enumerate}[leftmargin=1.4em]
\item For each factual rollout, enumerate a candidate pool of counterfactual interventions.
\item Mark each candidate with its transition label among $\{S\!\to\!S, S\!\to\!F, F\!\to\!S, F\!\to\!F\}$ and with a Boolean \texttt{success\_change} indicator.
\item Compute the endpoint displacement statistic \texttt{endpoint\_delta}. A candidate is considered \emph{informative} if it changes success, or if its endpoint displacement is above the 35th percentile of the non-zero endpoint displacements in the candidate pool.
\item Split informative candidates into \emph{change} and \emph{no-change} subsets.
\item Select approximately half the final queries from the change subset and the remainder from the no-change subset.
\item Inside each subset, perform round-robin selection over transition labels while enforcing at most one selected query per factual rollout (\texttt{factual\_id}).
\item If the target query budget is not yet filled, greedily back-fill from the remaining informative candidates while still enforcing unique factual rollouts.
\end{enumerate}

This sampler is intentionally not optimized to maximize the gap between models. Its purpose is to approximate a balanced, per-factual counterfactual evaluation set in which success-changing and non-changing queries are both represented.

\subsection{Main-task query statistics}

Table~\ref{tab:appendix-sampler} lists the realized query statistics of the formal MuJoCo Push and Door main experiments.

\begin{table}[t]
\caption{Formal balanced-sampler query statistics for the main MuJoCo tasks.}
\label{tab:appendix-sampler}
\centering
\small
\begin{tabular}{lcccccc}
\toprule
Task & Queries & Fact. Succ. & CF Succ. & Change Rate & Transition Counts & Window Mean \\
\midrule
Push & 32 & 0.2500 & 0.2500 & 0.5000 & $0/8/8/16$ & 2.91 \\
Door & 32 & 0.6875 & 0.3125 & 0.5000 & $8/14/2/8$ & 6.28 \\
\bottomrule
\end{tabular}
\end{table}

In Table~\ref{tab:appendix-sampler}, transition counts are ordered as
\[
S\!\to\!S \,/\, S\!\to\!F \,/\, F\!\to\!S \,/\, F\!\to\!F.
\]
For task-level non-monotonicity comparison, we also use the composite score
\[
\mathsf{NMS}=\rho_{\mathrm{flip}}\cdot \rho_{\mathrm{contact}},
\]
with $\rho_{\mathrm{flip}}$ the Context Flip Ratio and $\rho_{\mathrm{contact}}$ the contact ratio measured on the test split. Under this definition, Push has $\mathsf{NMS}\approx 0.175$ from $(0.4184, 0.4184)$, while Door has $\mathsf{NMS}\approx 0.696$ from $(0.7853,0.8869)$. We use this score only as an empirical summary of weak versus strong task non-monotonicity; it is not an additional assumption of the theory.
For the formal three-seed MuJoCo Push replication, the aggregated statistics are $96$ total queries, factual success rate $0.1979 \pm 0.0390$, counterfactual success rate $0.3229 \pm 0.0642$, success-change rate $0.5000 \pm 0.0000$, intervention-window mean $2.88 \pm 0.14$, and transition counts
\[
S\!\to\!S=1,\quad S\!\to\!F=18,\quad F\!\to\!S=30,\quad F\!\to\!F=47.
\]

\section{Implementation Details}

\subsection{Synthetic benchmark}

The synthetic benchmark is generated by \texttt{code/experiment\_1\_1/runner.py}. The full sweep uses:
\begin{itemize}[leftmargin=1.2em]
\item mechanism families: \texttt{global\_monotone}, \texttt{threshold\_flip}, and \texttt{smooth\_flip};
\item noise families: Gaussian, skewed, Gaussian mixture, and Student-$t$;
\item dimensions: $d \in \{3,5,10\}$;
\item training sizes: $N \in \{2000,10000,50000\}$.
\end{itemize}
This yields $3 \times 4 \times 3 \times 3 = 108$ configurations. For each configuration, the test size is \texttt{max(1000, N/5)} and the counterfactual evaluation size is \texttt{max(500, N/10)}. Seeds are assigned sequentially from $7$ upward across the sweep.

\paragraph{Continuous non-monotonicity bridge.}
The bridge sweep is generated by \texttt{code/experiment\_1\_2/run\_nonmonotone\_bridge.py}. It fixes $d=5$, train/test/counterfactual sizes $10000/2000/1000$, and evaluates $8$ control strengths
\[
0.00,\ 0.15,\ 0.30,\ 0.45,\ 0.60,\ 0.75,\ 0.90,\ 1.00
\]
across the same four exogenous noise families and five seeds $\{7,13,23,37,53\}$, for $8 \times 4 \times 5 = 160$ total runs. The bridge SCM calibrates a target sign-flip rate for each non-root mechanism, so the realized empirical synthetic non-monotonicity score rises smoothly from near $0$ to near $1$ rather than drifting unpredictably across seeds. For this bridge only, the synthetic score is
\[
\mathsf{NMS}_{\mathrm{synth}}
=
\frac{1}{d-1}\sum_{i=2}^d 2\min(r_i,1-r_i),
\]
where $r_i$ is the empirical negative-direction rate of mechanism $i$ on the test split. Thus $\mathsf{NMS}_{\mathrm{synth}}=0$ corresponds to a globally monotone orientation field and $\mathsf{NMS}_{\mathrm{synth}}=1$ to a maximally balanced parent-dependent sign flip.

\subsection{Physical datasets}

The physical experiments use the dataset size
\[
\texttt{train/test/cf} = 128/32/32
\]
for all formal MuJoCo main results and robustness runs. Here these counts refer to \emph{episodes}, not one-step supervised samples. Because the learned dynamics baselines are trained from per-step transitions extracted from each trajectory, the formal single-seed training split yields $17{,}289$ Push transitions and $10{,}936$ Door transitions for the feedforward and autoregressive regressors. The GRU baseline is trained on the same $128$ trajectories but still sees these $10^4$-scale time steps under masked autoregressive sequence supervision. We therefore interpret the physical comparison as a low-trajectory, state-based regime rather than a literal $128$-sample neural prediction problem.

\paragraph{Push environment.}
The Push configuration uses episode length $140$, physics step $1/120$, action repeat $4$, workspace $[-0.55,0.55]\times[-0.45,0.45]$, pusher radius $0.03$, box half-extent $0.035$, goal radius $0.10$, contact proximity threshold $0.085$, and contact offset threshold $0.010$. The formal MuJoCo Push main experiment fixes
\[
\texttt{contact\_forward\_gain}=1.0,\qquad
\texttt{context\_force\_gain}=0.3.
\]

\paragraph{Door environment.}
The Door configuration uses episode length $140$, physics step $1/120$, action repeat $4$, workspace $[-0.35,0.45]\times[-0.40,0.40]$, effector radius $0.025$, default door length $0.28$, goal tolerance $0.05$, contact proximity threshold $0.090$, and radial offset threshold $0.010$. The formal Door robustness conditions are:
\begin{itemize}[leftmargin=1.2em]
\item \textbf{base:} no additional noise;
\item \textbf{obs\_noise:} observation noise standard deviation $0.01$;
\item \textbf{dyn\_noise:} action execution noise standard deviation $0.04$ and torque noise standard deviation $0.08$.
\end{itemize}

\subsection{Physical CausalInverter instantiations}

In the current low-dimensional physical experiments, CausalInverter is instantiated as a structured regression model rather than a deep neural network. This is deliberate: the goal of these experiments is to isolate the structural contribution of inversion, orientation gating, and transport, not to maximize raw function approximation capacity.

\paragraph{Push model.}
The MuJoCo Push CausalInverter contains two stages.
\begin{enumerate}[leftmargin=1.4em]
\item A global linear backbone predicts the raw next-state box delta from handcrafted features using ridge regression with regularization $2\times 10^{-3}$.
\item A contact-gated residual model predicts corrections in a local contact-aligned frame. The global residual weights use weighted ridge regression with regularization $10^{-2}$, and mode-specific residual weights use regularization $2\times 10^{-2}$ for modes with at least $24$ training samples.
\end{enumerate}
The residual sample weight is
\[
1 + 2.5 \cdot \texttt{contact}
 + 1.5 \cdot \texttt{sign\_strength}
 + 0.75 \cdot \texttt{mode\_boost},
\]
so that contact-heavy and branch-sensitive samples receive larger emphasis. During rollout, residual corrections are clipped coordinate-wise before reconstruction, and blended with the backbone prediction using a context-dependent coefficient $\alpha \in [0,0.85]$.

\paragraph{Door model.}
The MuJoCo Door CausalInverter uses a mode-conditioned linear model in a local door-contact frame. A global ridge regressor is fit with regularization $5\times 10^{-3}$, and mode-specific regressors use regularization $10^{-2}$ for modes with at least $24$ samples. The \texttt{use\_gate}, \texttt{use\_transport}, and \texttt{use\_inverse} ablations correspond exactly to the reported model deletions in the paper:
\begin{itemize}[leftmargin=1.2em]
\item \texttt{use\_gate=False}: remove context-dependent orientation gating;
\item \texttt{use\_transport=False}: remove factual-prefix-preserving counterfactual transport and restart from the full rollout;
\item \texttt{use\_inverse=False}: weaken factual inversion by reconstructing a canonicalized local start observation.
\end{itemize}

\paragraph{Stronger aligned baselines.}
The main MuJoCo tables additionally include five stronger state-based baselines under the same observed-state and counterfactual protocol: an \textbf{MLP World Model}, an \textbf{Autoregressive Dynamics} model, a \textbf{GRU Sequence Model}, a compact \textbf{Transformer Sequence Model}, and a \textbf{Conditional Flow Dynamics} baseline. All five are trained as residual predictors on top of the corresponding linear backbone rather than as pixel-based policies or planners. The MLP baseline uses a feedforward residual regressor, the autoregressive baseline predicts the next-state residual coordinates sequentially, and the flow baseline uses a conditional affine-coupling model trained by likelihood and decoded with zero latent noise at test time. The GRU and Transformer baselines are implemented as masked autoregressive state-sequence predictors: at training time each step predicts its residual target from prefix-visible sequence features only, with strict temporal masking for the Transformer and padding masks for variable-length trajectories. At counterfactual test time, they use the factual prefix only to initialize the recurrent/sequence cache up to the intervention step, and then roll out the post-intervention suffix autoregressively from the intervention state using previously generated states together with the counterfactual action suffix, without access to future observations. Under the balanced sampler, the pre-intervention action prefix is unchanged, so the factual and counterfactual prefixes coincide in action space as well. These are intentionally compact low-data predictors rather than large-scale sequence or vision models, since the goal of the physical section is to test whether the structural bias survives against stronger function approximators under the same limited-state supervision. These stronger baselines are reported on the main balanced splits; the multi-seed Push replication and the Door robustness suite keep the earlier linear/context stress-test design to control runtime and preserve protocol continuity with the earlier ablation analysis. In particular, the Door seed-wise head-to-head is kept to Transformer, conditional flow, and our model because it is meant to compare three distinct rollout families under repeated resampling, while the feedforward MLP is already documented on the formal main split and in the paired significance table.

\subsection{Evaluation metrics}

Push reports forward and counterfactual box-trajectory errors, endpoint errors, final success accuracy, success-change accuracy, and a mode-based latent-consistency proxy. Door reports forward and counterfactual door-angle errors, endpoint errors, success accuracy, success-change accuracy, latent consistency, early latent consistency, contact accuracy, context-sign accuracy, and mode-distribution $\ell_1$ mismatch. We treat the endpoint, stage-wise, mode-conditioned, contact, context-sign, and mode-distribution metrics as \emph{counterfactual explanation-quality diagnostics}: they evaluate not only whether a counterfactual ends in success or failure, but also whether the model preserves the physically meaningful contact mode, branch, and local trajectory structure that make the explanation plausible. Stage-wise and mode-conditioned errors are computed from the same counterfactual rollouts by aligning each time index with a phase label.

\subsection{Artifact locations}

The main artifact directories used in this paper are:
\begin{itemize}[leftmargin=1.2em]
\item \texttt{code/outputs/exp1\_1/} for the synthetic benchmark;
\item \texttt{code/outputs/exp1\_2\_nonmonotone\_bridge/} for the continuous synthetic bridge sweep;
\item \texttt{code/outputs/exp4\_3\_push\_mujoco\_balanced/} for the formal single-seed MuJoCo Push main result;
\item \texttt{code/outputs/exp4\_3\_push\_mujoco\_balanced\_multiseed/} for the formal three-seed MuJoCo Push replication;
\item \texttt{code/outputs/exp4\_3\_door\_mujoco/} for the formal MuJoCo Door main result and ablations;
\item \texttt{code/outputs/exp4\_3\_door\_mujoco\_robustness/} for the Door multi-seed robustness results.
\end{itemize}
Each dataset directory contains \texttt{metadata.json}; each result directory contains raw \texttt{json} summaries and markdown tables.

\subsection{Additional diagnostic tables}

\begin{table}[t]
\caption{Family-level comparison of the structural assumptions that matter for counterfactual semantics.}
\label{tab:related-comparison}
\centering
\scriptsize
\begin{tabular}{@{}lcccc@{}}
\toprule
Family & Local inv. & Stable $U$ align. & CF ident. & Remark \\
\midrule
Monotone Tri-SCMs & Yes & Yes & Yes & Needs global mono. \\
Exogenous-isomorphic SCMs & Varies & By def. & Yes & Semantic criterion \\
Embodied world models & Optional & No & No guarantee & Predictive focus \\
CausalInverter (ours) & Yes & Yes & Yes in \NM & Context flips \\
\bottomrule
\end{tabular}
\end{table}

\begin{table}[t]
\caption{Paired significance tests for the main synthetic benchmark. Each row compares our model against a matched baseline on the same synthetic configurations within a mechanism family. Negative mean differences indicate lower CF-MSE for our model.}
\label{tab:appendix-synth-significance}
\centering
\small
\begin{tabular}{llccc}
\toprule
Mechanism & Baseline & Mean Diff. $\downarrow$ & 95\% CI & Wilcoxon $p$ \\
\midrule
Threshold Flip & TM-SCM Quantile & $-0.5674$ & [$-0.7217$, $-0.4285$] & $1.46\times 10^{-11}$ \\
Threshold Flip & BSCM-Flow Contextual & $-0.5685$ & [$-0.7311$, $-0.4263$] & $1.46\times 10^{-11}$ \\
Smooth Flip & TM-SCM Quantile & $-0.4968$ & [$-0.6413$, $-0.3689$] & $1.46\times 10^{-11}$ \\
Smooth Flip & BSCM-Flow Contextual & $-0.4980$ & [$-0.6434$, $-0.3725$] & $1.46\times 10^{-11}$ \\
\bottomrule
\end{tabular}
\end{table}

\begin{table}[t]
\caption{Continuous synthetic bridge sweep. As the calibrated synthetic non-monotonicity score increases, the counterfactual advantage of our model over TM-SCM grows systematically.}
\label{tab:appendix-synth-bridge}
\centering
\small
\begin{tabular}{lccccc}
\toprule
Strength & $\mathsf{NMS}_{\mathrm{synth}}$ & TM-SCM & Context Flow & Ours & Gain over TM-SCM \\
\midrule
0.00 & 0.000 & 0.0008 & 0.0014 & \textbf{0.0000} & 0.0008 \\
0.15 & 0.147 & 0.3541 & 0.3562 & \textbf{0.2898} & 0.0643 \\
0.30 & 0.298 & 0.5839 & 0.5892 & \textbf{0.4847} & 0.0992 \\
0.45 & 0.447 & 0.7850 & 0.7936 & \textbf{0.3723} & 0.4127 \\
0.60 & 0.600 & 0.9791 & 0.9914 & \textbf{0.2768} & 0.7023 \\
0.75 & 0.745 & 1.0342 & 1.0550 & \textbf{0.2345} & 0.7997 \\
0.90 & 0.894 & 1.0536 & 1.0821 & \textbf{0.2140} & 0.8395 \\
1.00 & 0.978 & 1.0464 & 1.0785 & \textbf{0.2055} & 0.8409 \\
\bottomrule
\end{tabular}
\end{table}

\begin{table}[t]
\caption{Paired significance tests for the synthetic ablation benchmark. The dominant ablation effects on non-monotone mechanisms come from removing the direction gate or transport regularization; context-amplitude and cycle terms are not consistently significant under this sweep.}
\label{tab:appendix-synth-ablation-significance}
\centering
\small
\begin{tabular}{llccc}
\toprule
Mechanism & Comparison & Mean Diff. $\downarrow$ & 95\% CI & Wilcoxon $p$ \\
\midrule
Threshold Flip & Ours vs.\ w/o Gate & $-0.4043$ & [$-0.5837$, $-0.2457$] & $2.86\times 10^{-6}$ \\
Threshold Flip & Ours vs.\ w/o Transport & $-0.0541$ & [$-0.0762$, $-0.0330$] & $2.38\times 10^{-5}$ \\
Threshold Flip & Ours vs.\ w/o Context Amp. & $-0.0002$ & [$-0.0018$, $0.0016$] & $0.3108$ \\
Smooth Flip & Ours vs.\ w/o Gate & $-0.2589$ & [$-0.3677$, $-0.1619$] & $6.68\times 10^{-6}$ \\
Smooth Flip & Ours vs.\ w/o Transport & $-0.0381$ & [$-0.0597$, $-0.0189$] & $1.61\times 10^{-4}$ \\
Smooth Flip & Ours vs.\ w/o Context Amp. & $-0.0024$ & [$-0.0054$, $0.0005$] & $0.1305$ \\
\bottomrule
\end{tabular}
\end{table}

\begin{table}[t]
\caption{Effective physical training sample counts and fitted parameter scales for the main low-data state baselines. Episode counts are not one-step supervision counts: feedforward, autoregressive, and flow baselines train on per-step transitions, while recurrent and Transformer baselines consume the same trajectories under masked autoregressive sequence supervision.}
\label{tab:appendix-physical-scale}
\centering
\scriptsize
\begin{tabular}{lccccccccc}
\toprule
Task & Episodes & Steps & Linear & MLP & AR & GRU & Transf. & Flow & Ours \\
\midrule
Push & 128 & 17,289 & 138 & 34,384 & 138,000 & 68,496 & 79,696 & 117,666 & 714 \\
Door & 128 & 10,936 & 18 & 30,740 & 40,852 & 60,948 & 78,164 & 99,098 & 190 \\
\bottomrule
\end{tabular}
\end{table}

\begin{table}[t]
\caption{Paired significance tests for the main MuJoCo physical comparisons. Negative mean CF-MSE differences indicate lower error for our model; positive event differences indicate higher counterfactual event accuracy for our model.}
\label{tab:appendix-physical-significance}
\centering
\scriptsize
\begin{tabular}{llcccc}
\toprule
Task & Comparison & Mean diff. & 95\% CI & Wilcoxon $p$ & McNemar $p$ \\
\midrule
Push & Ours vs.\ Linear Dyn. & $0.0011$ & [$0.0005$, $0.0017$] & $9.93\times 10^{-1}$ & $9.89\times 10^{-1}$ \\
Push & Ours vs.\ MLP World Model & $-0.0043$ & [$-0.0062$, $-0.0027$] & $3.93\times 10^{-8}$ & $9.89\times 10^{-1}$ \\
Push & Ours vs.\ Transformer Seq. & $0.0006$ & [$0.0001$, $0.0012$] & $9.51\times 10^{-1}$ & $8.87\times 10^{-1}$ \\
Door & Ours vs.\ MLP World Model & $-0.0287$ & [$-0.0926$, $0.0282$] & $1.75\times 10^{-1}$ & $1.00$ \\
Door & Ours vs.\ Transformer Seq. & $-0.0836$ & [$-0.1532$, $-0.0186$] & $3.39\times 10^{-2}$ & $1.00$ \\
Door & Ours vs.\ Cond.\ Flow Dyn. & $0.0486$ & [$0.0130$, $0.0859$] & $9.85\times 10^{-1}$ & $3.12\times 10^{-2}$ \\
\bottomrule
\end{tabular}
\end{table}

\begin{table}[t]
\caption{Door three-seed head-to-head comparison for the strongest learned baselines under the same balanced counterfactual protocol. Means and standard deviations are over seeds.}
\label{tab:appendix-door-head2head}
\centering
\scriptsize
\begin{tabular}{lcccc}
\toprule
Model & Fwd. Door-MSE & CF Door-MSE & CF Success Acc. & CF Latent Cons. \\
\midrule
Transformer Seq. & $0.0912 \pm 0.0580$ & $0.3740 \pm 0.1288$ & $0.5521 \pm 0.3231$ & $0.3058 \pm 0.0180$ \\
Cond.\ Flow Dyn. & $0.0847 \pm 0.0342$ & $\mathbf{0.0500 \pm 0.0139}$ & $0.7083 \pm 0.1031$ & $0.3405 \pm 0.0550$ \\
Ours & $0.1916 \pm 0.0235$ & $0.1087 \pm 0.0090$ & $\mathbf{0.9583 \pm 0.0295}$ & $0.2468 \pm 0.0244$ \\
\bottomrule
\end{tabular}
\end{table}

\begin{table}[t]
\caption{Door seed-wise head-to-head results. Each entry reports `CF Door-MSE / CF Success Acc.' for one seed under the same balanced protocol.}
\label{tab:appendix-door-head2head-seedwise}
\centering
\scriptsize
\begin{tabular}{lccc}
\toprule
Seed & Transformer Seq. & Cond.\ Flow Dyn. & Ours \\
\midrule
7 & $0.1918 / 1.0000$ & $0.0596 / 0.8438$ & $0.1082 / 1.0000$ \\
17 & $0.4672 / 0.2500$ & $0.0600 / 0.6875$ & $0.1199 / 0.9375$ \\
27 & $0.4630 / 0.4062$ & $0.0304 / 0.5938$ & $0.0979 / 0.9375$ \\
\bottomrule
\end{tabular}
\end{table}

\begin{table}[t]
\caption{Balanced MuJoCo Push ablation summary under the same protocol as the main Push result. The continuous CF geometry cleanly separates transport/inverse components from the context baseline, but does not induce a stable gate ranking or event-level ordering.}
\label{tab:appendix-push-ablation}
\centering
\small
\begin{tabular}{lccc}
\toprule
Model & CF Box-MSE & CF Success Acc. & CF Latent Cons. \\
\midrule
Context Dynamics & 0.0108 & 0.7188 & 0.4016 \\
Ours w/o Gate & 0.0027 & 0.4688 & 0.2162 \\
Ours w/o Transport & 0.0055 & 0.7500 & 0.6538 \\
Ours w/o Inverse Stability & 0.0038 & 0.7500 & 0.6538 \\
Ours & 0.0027 & 0.5625 & 0.2145 \\
\bottomrule
\end{tabular}
\end{table}

\begin{table}[t]
\caption{Door explanation-quality diagnostics beyond CF-MSE and event accuracy. These quantities evaluate whether the predicted counterfactual preserves the physically meaningful contact pattern and local phase geometry, but they do not perfectly align with event-level branch recovery once stronger low-data predictors are included.}
\label{tab:appendix-door-diagnostics}
\centering
\scriptsize
\begin{tabular}{lccccc}
\toprule
Model & CF Endpt. $\downarrow$ & Contact $\uparrow$ & Context $\uparrow$ & Early Latent $\uparrow$ & Mode $\ell_1 \downarrow$ \\
\midrule
Linear Dynamics & 0.3539 & 0.5615 & 0.2964 & 0.5472 & 1.2911 \\
Context Dynamics & 0.4847 & 0.4316 & 0.3751 & 0.5307 & 1.3594 \\
MLP World Model & 0.3103 & 0.6134 & 0.5387 & 0.3908 & 1.2917 \\
Autoregressive Dyn. & 0.2797 & 0.8314 & 0.6785 & 0.3712 & 0.8301 \\
GRU Sequence Model & 0.8725 & 0.4287 & 0.3565 & 0.5147 & 1.3611 \\
Transformer Seq. & 0.2220 & 0.4392 & 0.2989 & 0.4569 & 1.3771 \\
Cond. Flow Dyn. & \textbf{0.1831} & \textbf{0.6765} & 0.4941 & 0.5073 & 1.0264 \\
Ours & 0.2990 & 0.5933 & 0.2608 & 0.3933 & 1.4466 \\
\bottomrule
\end{tabular}
\end{table}

\section{Additional Experimental Figures}

This section provides figure-level interpretations that complement the concise captions. The emphasis is on how each figure supports the main claims of the paper: the theoretical role of exogenous alignment, the boundary-aware empirical reading of Push versus Door, and the module-level behavior of CausalInverter.

\begin{figure}[H]
\centering
\includegraphics[width=\linewidth]{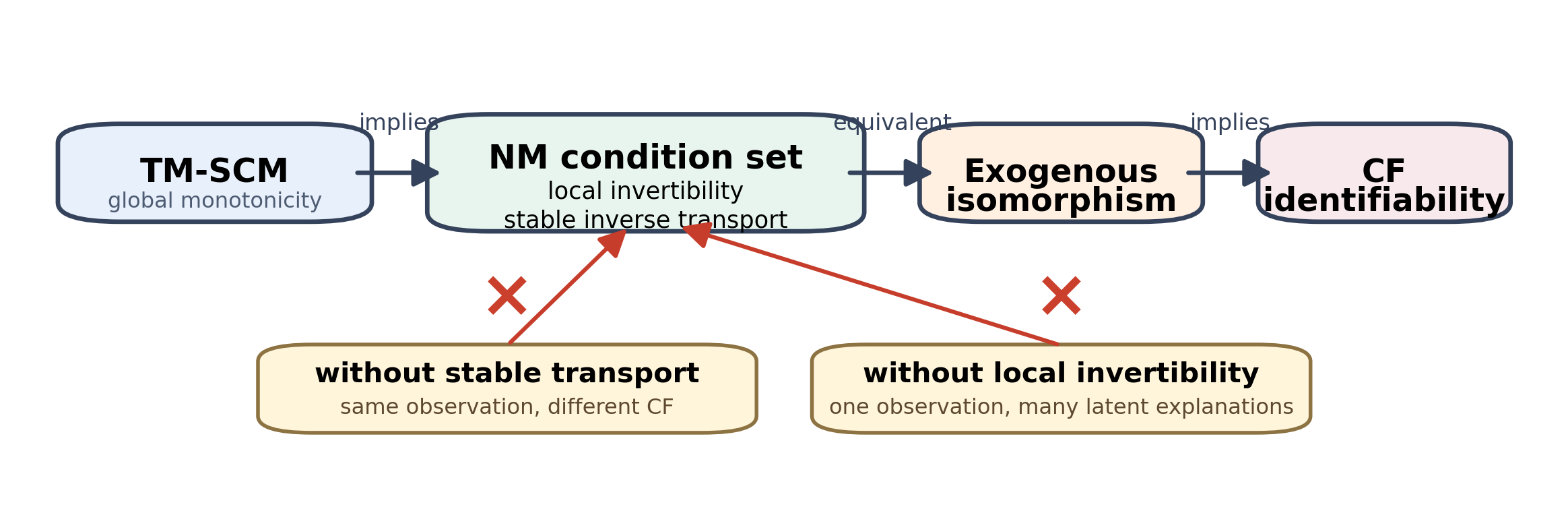}
\caption{The logical structure of the theory.}
\label{fig:theory-map}
\end{figure}

Figure~\ref{fig:theory-map} should be read as the proof roadmap. It separates three logically distinct layers: local mechanism invertibility, cross-context stability of inverse transport, and the semantic consequence of exogenous isomorphism. This figure is useful because the contribution is not simply another sufficient condition for counterfactual identifiability. Rather, the paper shows where the actual obstruction lies once triangular recursion is fixed: if inverse transport remains context-dependent, observationally equivalent models can relabel latent worlds differently and therefore disagree on counterfactuals. The diagram makes this dependency chain explicit.

\begin{figure}[H]
\centering
\includegraphics[width=\linewidth]{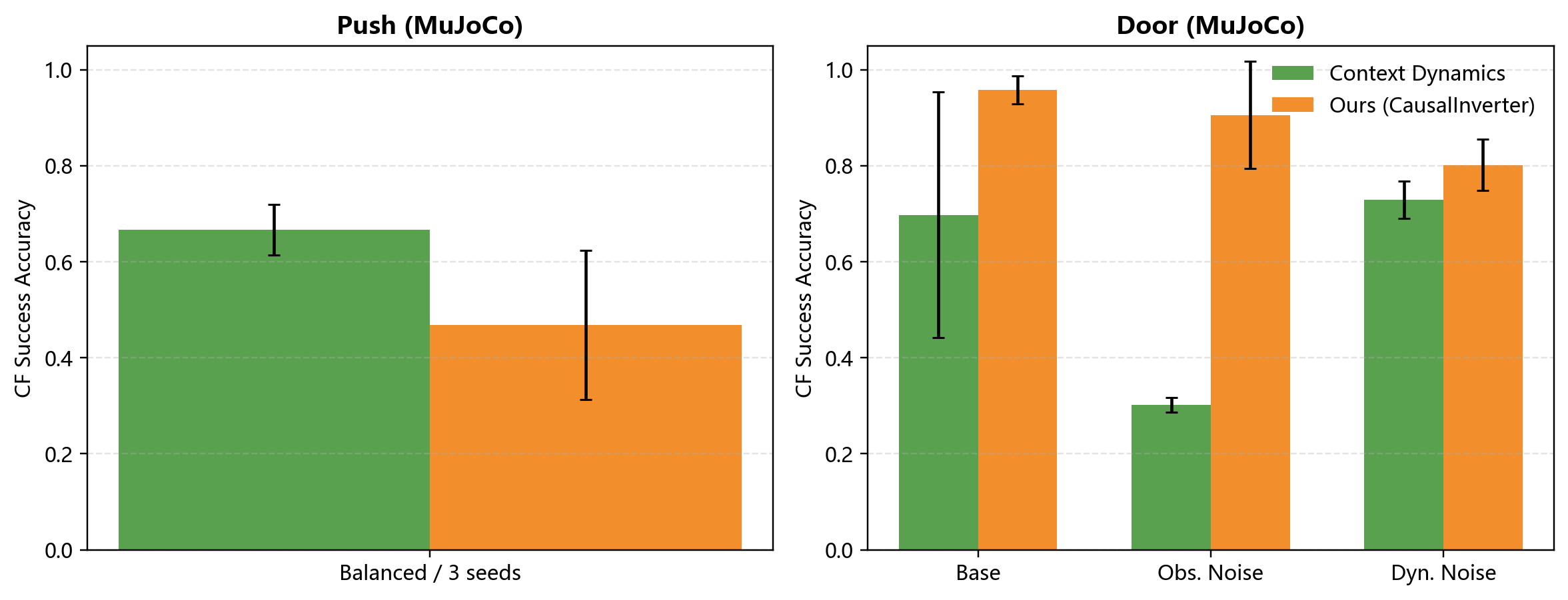}
\caption{Stability results beyond the single main split.}
\label{fig:robustness}
\end{figure}

Figure~\ref{fig:robustness} summarizes the multi-seed robustness experiments. It should be interpreted column by column rather than only by aggregate ranking. In Push, the most stable advantage of our method is on continuous geometry: counterfactual trajectory errors remain competitive across seeds, but event-level metrics fluctuate more strongly, reflecting the weak-\(\mathsf{NMS}\) regime discussed in the main text. In Door, the picture is different. The counterfactual success advantage remains comparatively stable across seeds and noise conditions, which is why Door is the main physical evidence for the theory. The figure therefore reinforces the paper's boundary-aware reading: the proposed structural bias is most decisive when non-monotone branching is sufficiently strong.

All figures below use English-only titles, axis labels, legends, and annotations.

\begin{figure}[H]
\centering
\includegraphics[width=0.48\linewidth]{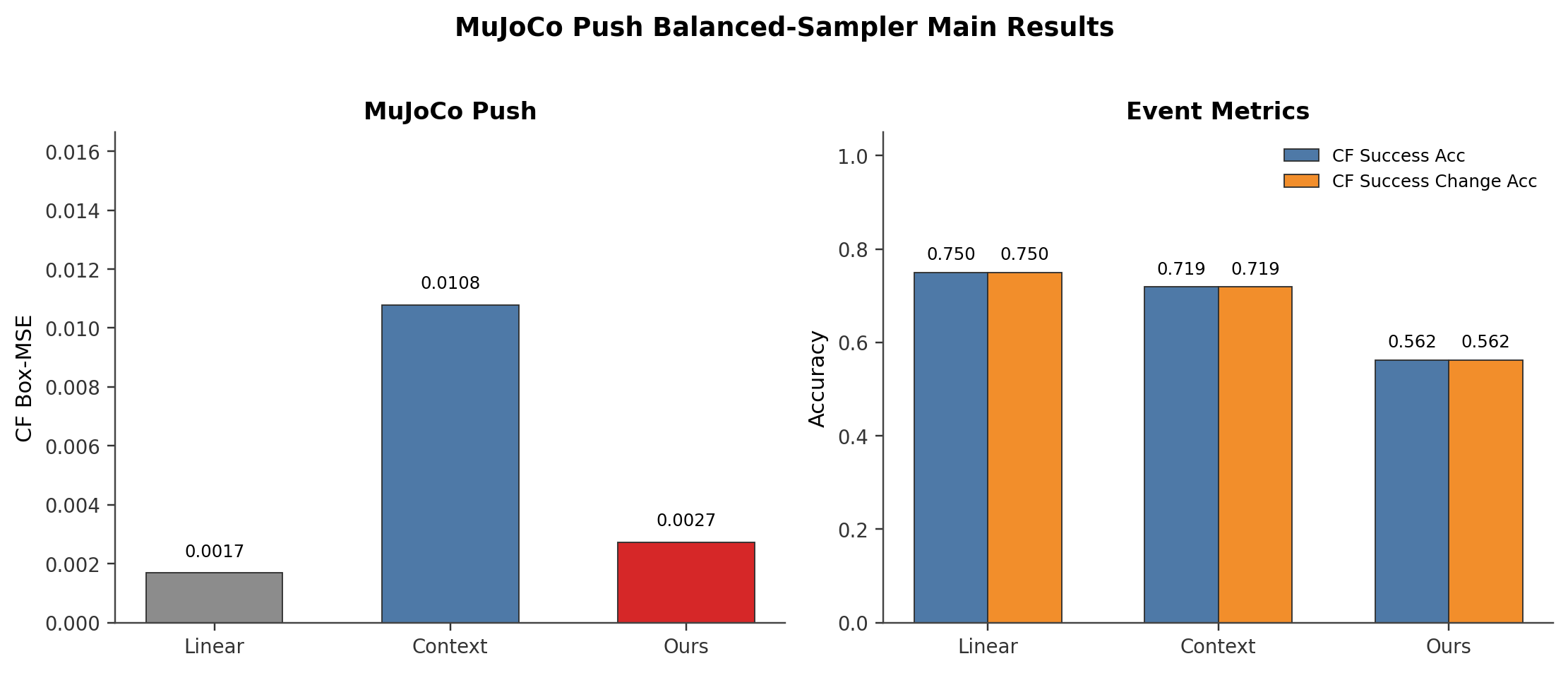}\hfill
\includegraphics[width=0.48\linewidth]{exp4_3_door_main_results_mujoco.png}
\caption{Per-task main result visualizations.}
\label{fig:appendix-main-physical}
\end{figure}

Figure~\ref{fig:appendix-main-physical} expands the main physical story into a more readable representative-baseline comparison. The Push panel should be read as evidence of a genuine boundary case: our method does not separate cleanly from stronger low-data predictors on either continuous or event metrics. The Door panel should be read differently. There, the single split separates three behaviors: Transformer and our model both recover the event labels perfectly but with very different threshold margins, whereas conditional flow achieves the lowest scalar angle error while missing a nontrivial fraction of event labels. The contrast between the two panels is central to the paper: Push is not presented as a contradiction, but as an explicit boundary case.

\begin{figure}[H]
\centering
\includegraphics[width=\linewidth]{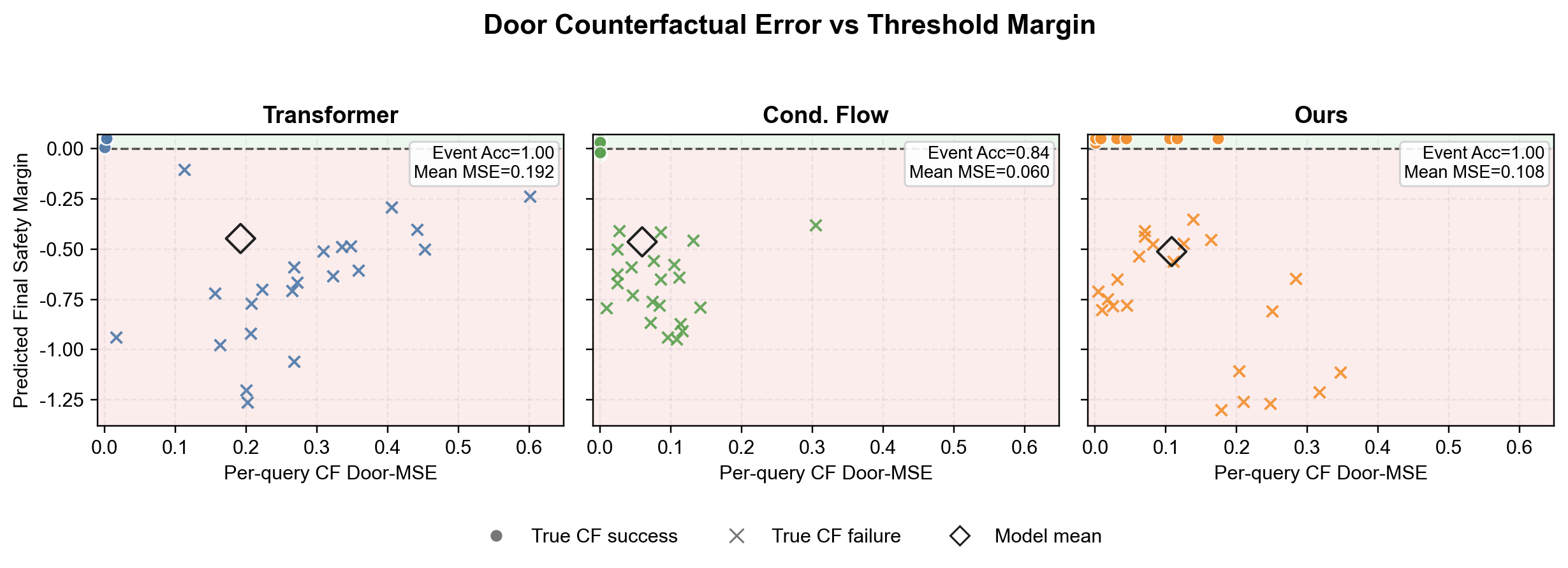}
\caption{Door per-query continuous error versus final threshold margin for three representative learned models. The horizontal line marks the success threshold.}
\label{fig:appendix-door-threshold}
\end{figure}

Figure~\ref{fig:appendix-door-threshold} clarifies the seemingly counterintuitive Door result. The reported `CF Success Acc.` is \emph{event-label accuracy}, not the raw fraction of successful openings; on the formal Door split, only $10/32$ counterfactual queries are true-success cases and the remaining $22/32$ are true-failure cases. Transformer and CausalInverter both attain event accuracy $1.0$ because they place every success query above the threshold and every failure query below it, but they do so with very different safety margins. Transformer often solves the success subset with only a narrow positive buffer (minimum margin $0.0048$) and exhibits the widest per-query CF-MSE spread (std.\ $0.1652$). CausalInverter keeps a larger success-floor margin ($0.0300$) and lowers mean CF Door-MSE to $0.1082$. Conditional flow illustrates the complementary failure mode: it attains the lowest mean CF Door-MSE ($0.0596$), but some success queries cross below threshold (minimum margin $-0.0311$), reducing event accuracy to $0.8438$. This is the intended interpretation of the Door result: continuous error alone does not determine counterfactual branch correctness, and the structured inverse improves branch stability rather than universally minimizing every scalar proxy.

\begin{figure}[H]
\centering
\includegraphics[width=0.48\linewidth]{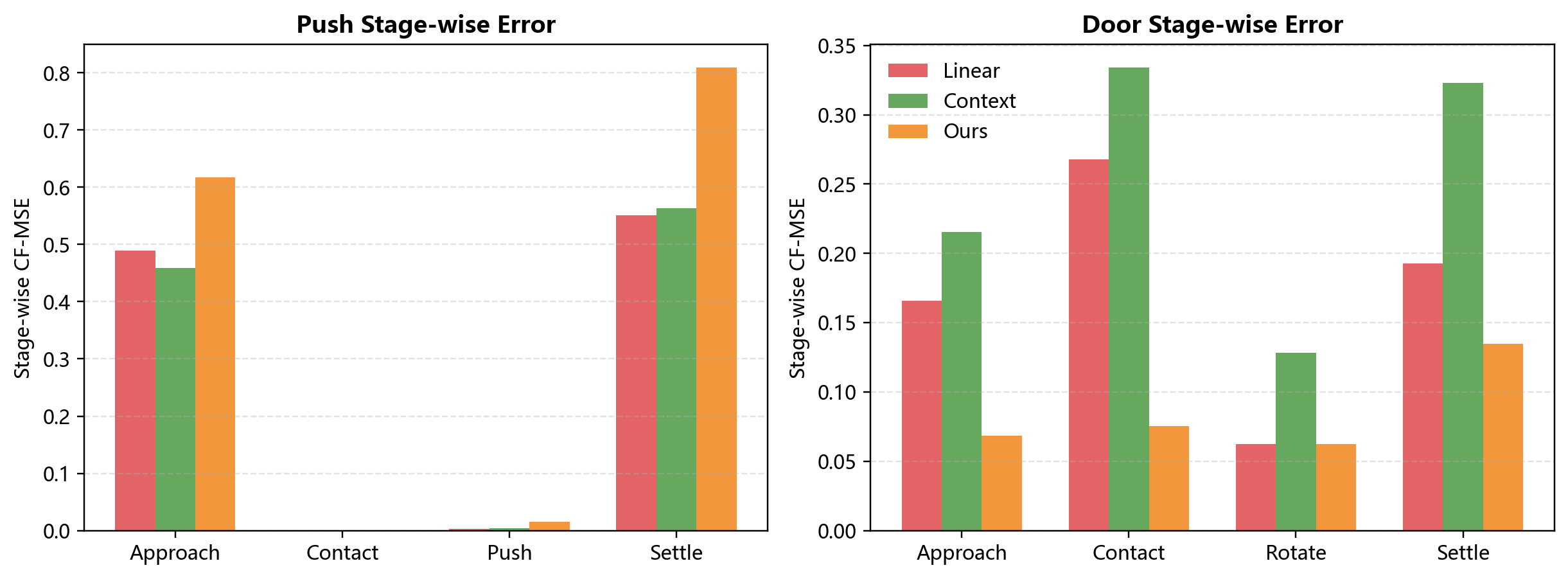}\hfill
\includegraphics[width=0.48\linewidth]{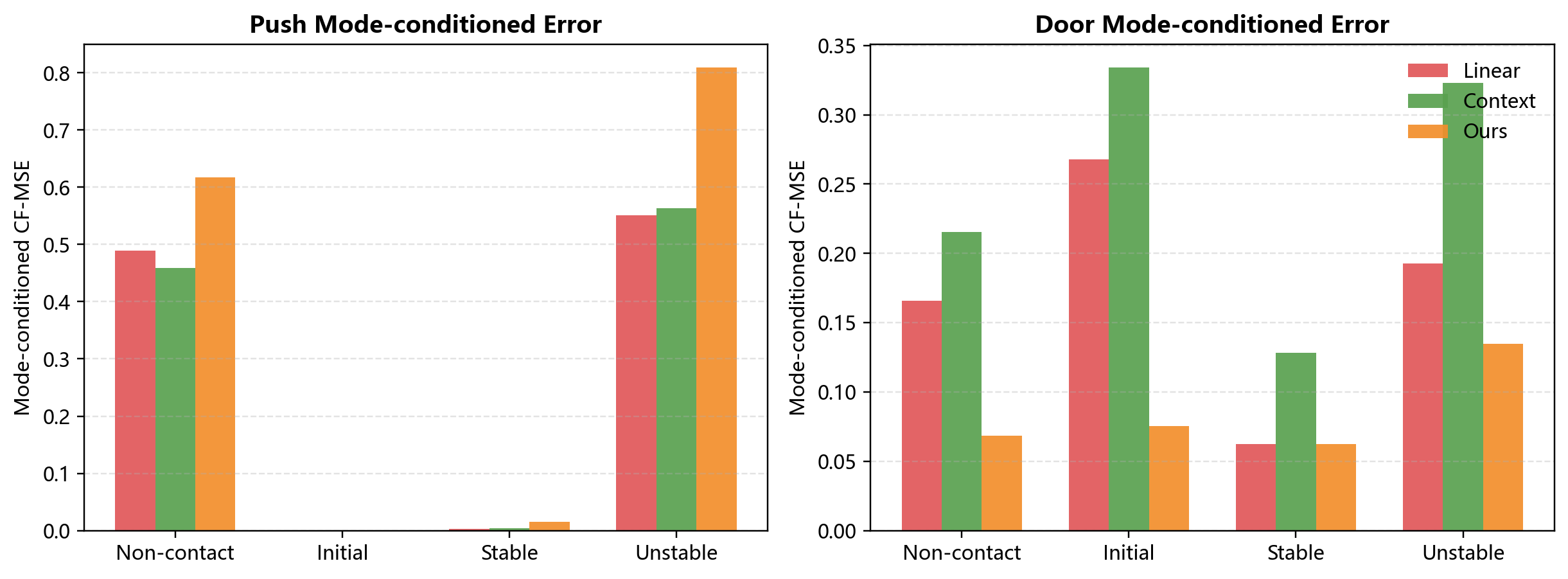}
\caption{Stage-wise and mode-conditioned diagnostics.}
\label{fig:appendix-stage-mode}
\end{figure}

Figure~\ref{fig:appendix-stage-mode} explains \emph{where} the gains come from. The stage-wise panel decomposes error over the temporal phases of the interaction, so the reader can check whether improvement is concentrated only in easy phases or persists in the contact-heavy part of the rollout. The mode-conditioned panel instead groups counterfactual error by interaction regime, which is especially informative for Door because the opening behavior contains clearer contact-dependent branching. Together, these plots show that the strongest benefits of our method are not uniformly distributed over the trajectory. They appear most clearly in the phases and modes where stable exogenous alignment matters most for preserving the correct counterfactual branch.

\begin{figure}[H]
\centering
\includegraphics[width=\linewidth]{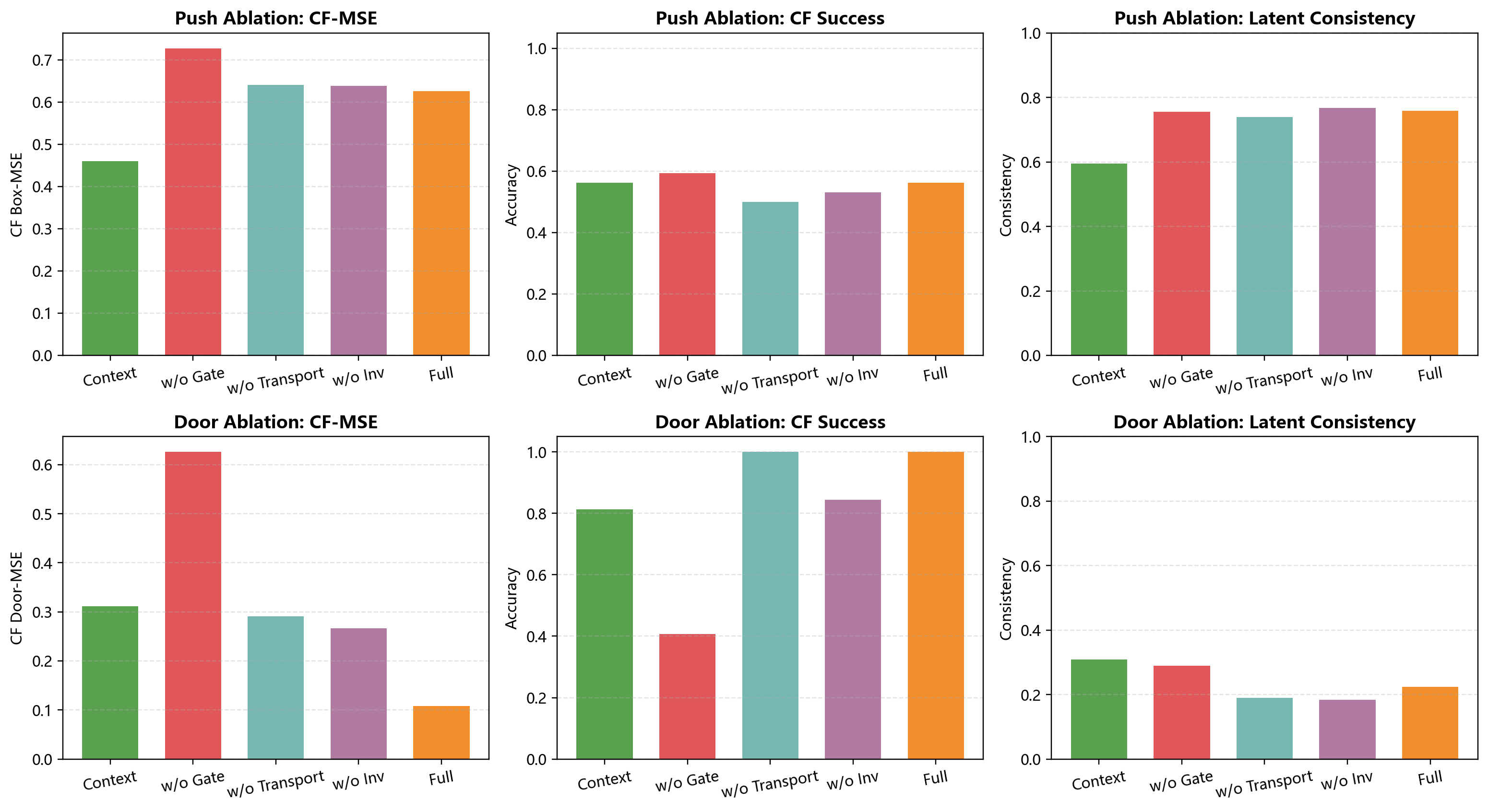}
\caption{Physical ablations for CausalInverter.}
\label{fig:appendix-physical-ablation}
\end{figure}

Figure~\ref{fig:appendix-physical-ablation} should be read as a module-separation study. On Door, removing the direction gate causes the most severe degradation, while removing transport stability or inverse stability also hurts but less dramatically. This is consistent with the theory: once the physical task contains clear context-dependent orientation changes, the model must represent those switches explicitly rather than average them away. Push exhibits a weaker ranking. There, transport and inverse components still help continuous geometry, but the gate does not create the same clean ordering. This asymmetry is informative rather than problematic, because it matches the lower non-monotonicity strength of Push.

\begin{figure}[H]
\centering
\includegraphics[width=0.48\linewidth]{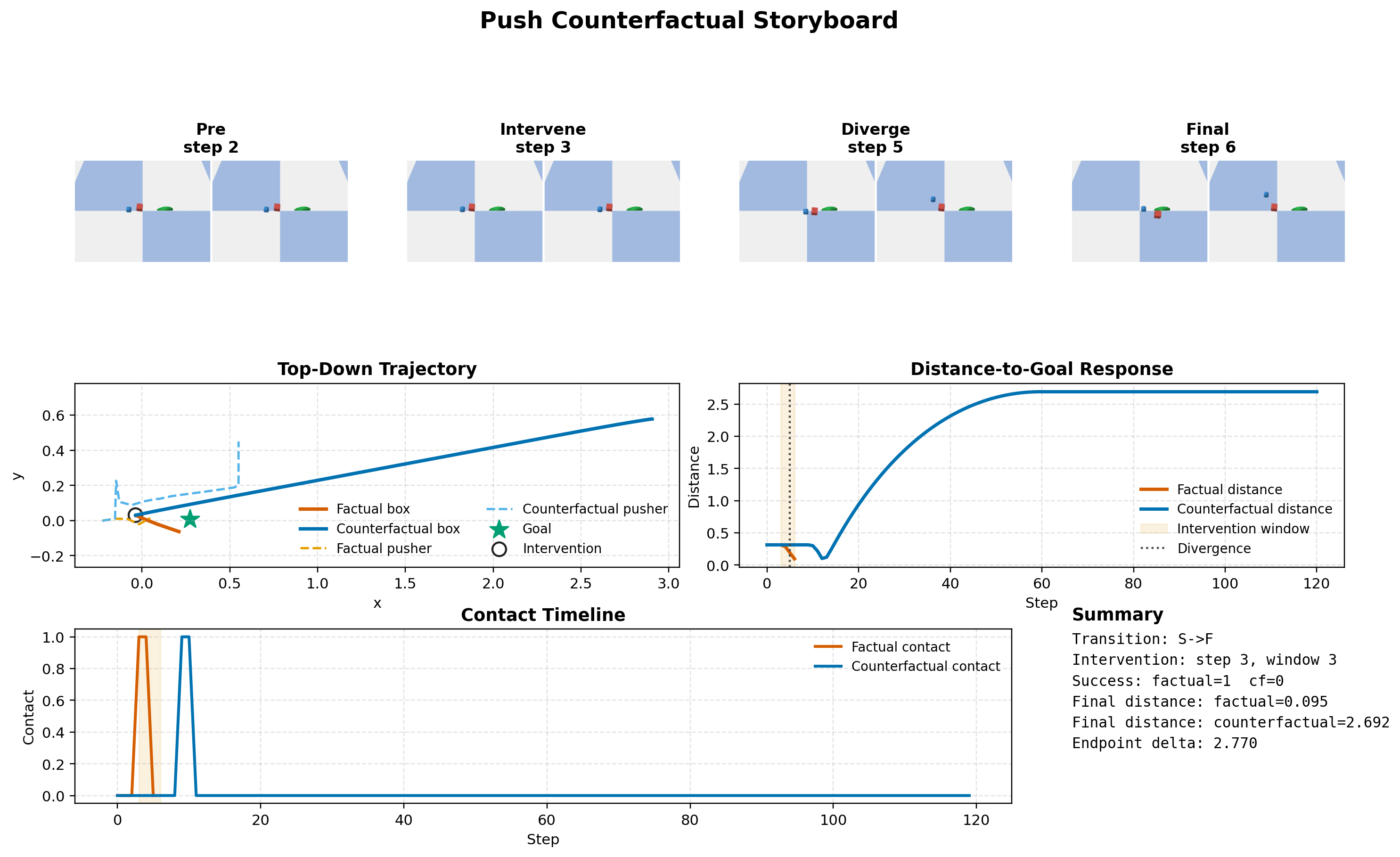}\hfill
\includegraphics[width=0.48\linewidth]{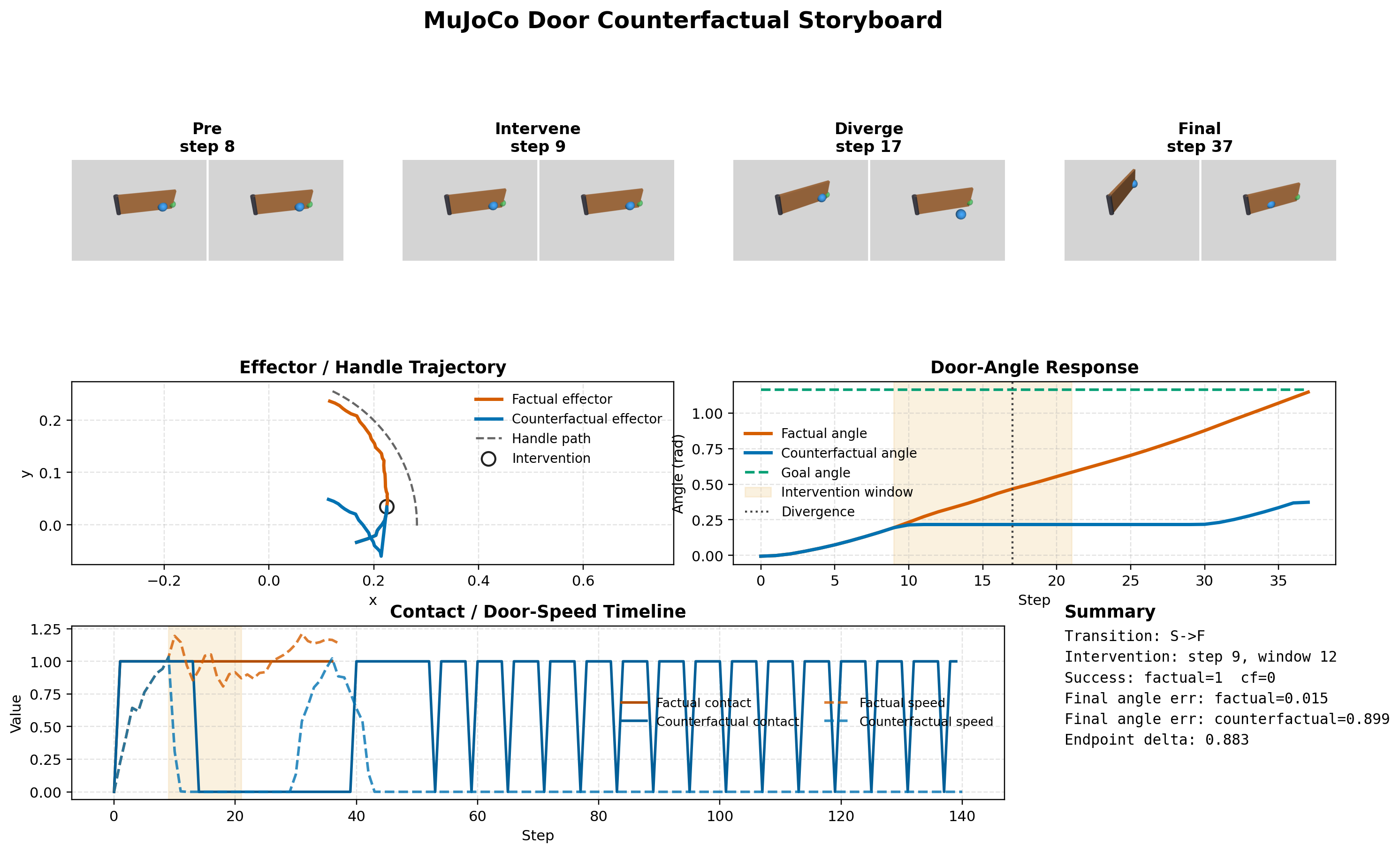}
\caption{Counterfactual storyboards for Push and Door.}
\label{fig:appendix-cf-storyboards}
\end{figure}

Figure~\ref{fig:appendix-cf-storyboards} provides trajectory-level intuition for the counterfactual protocol. Each storyboard starts from the same factual rollout and then shows how the trajectory changes under a counterfactual intervention while keeping the abducted latent background fixed. In Push, the reader should focus on how the box path shifts under the altered action even when the final success label may remain sensitive. In Door, the storyboard more clearly shows branch separation: after intervention, the door-angle evolution, contact pattern, and final opening outcome diverge in a way that is visually easier to distinguish. These examples help connect the formal abduction--action--prediction semantics to an intuitive physical picture.

\begin{figure}[H]
\centering
\includegraphics[width=0.48\linewidth]{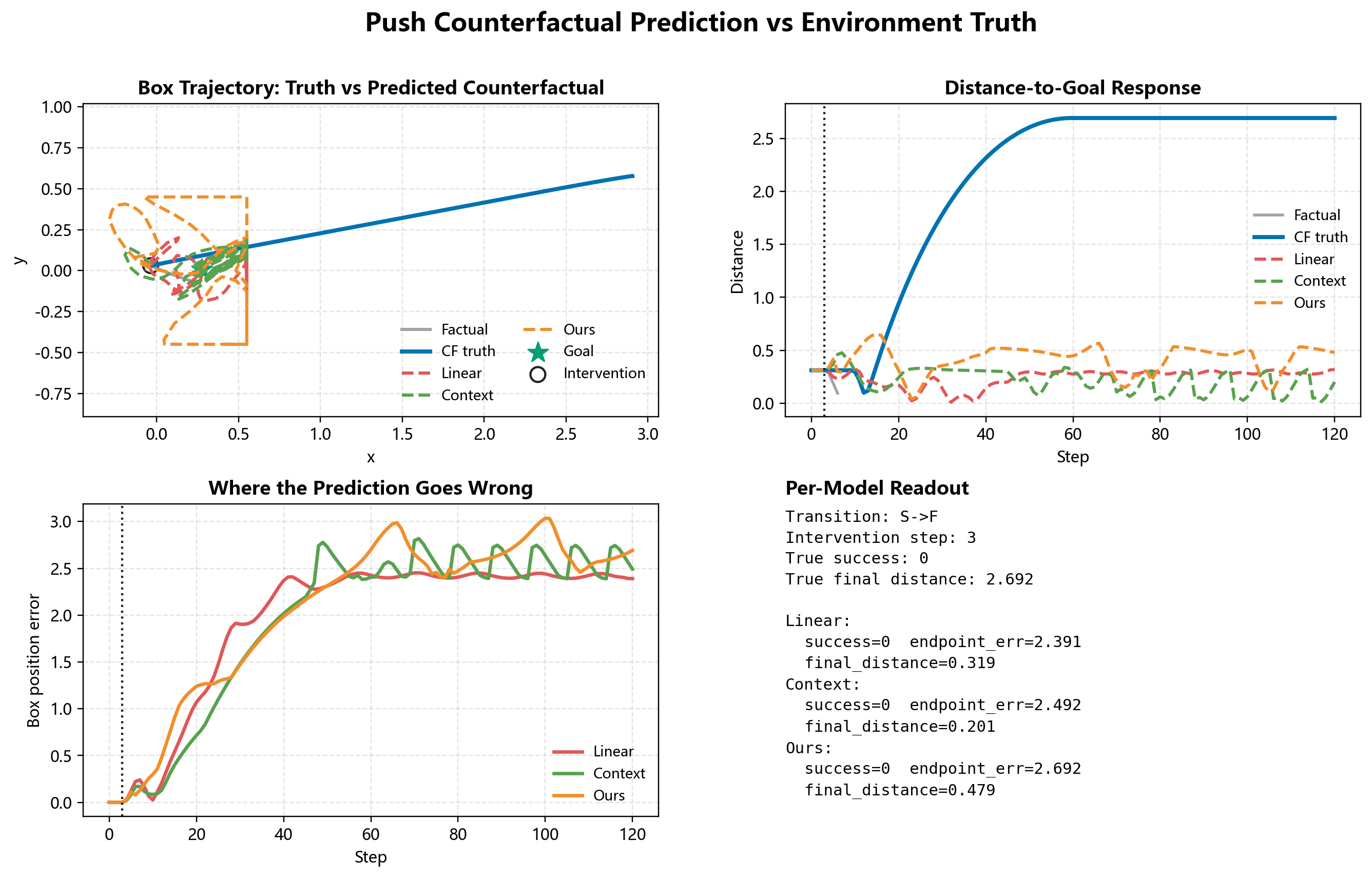}\hfill
\includegraphics[width=0.48\linewidth]{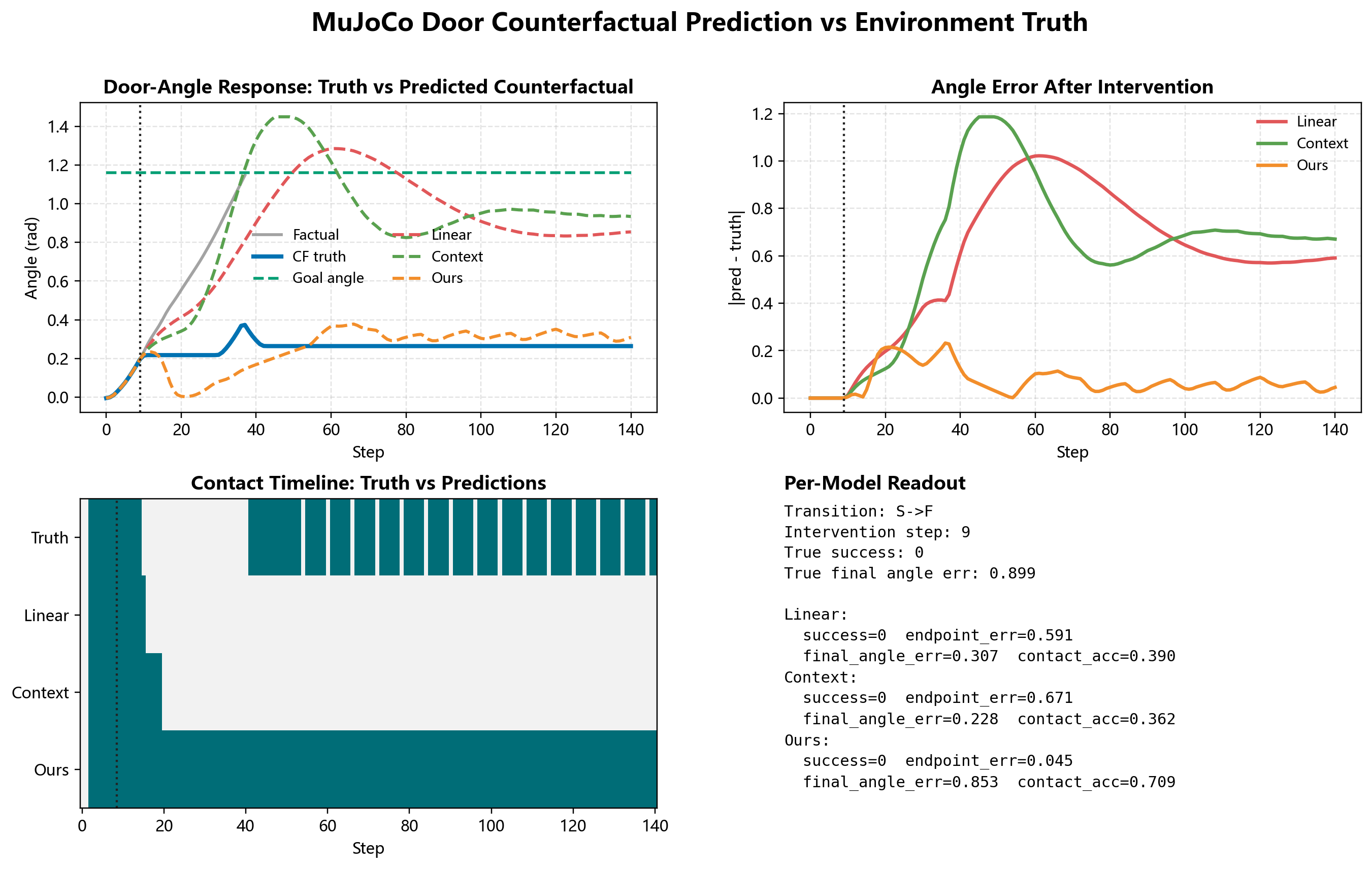}
\caption{Predicted versus environment-truth counterfactuals.}
\label{fig:appendix-pred-vs-truth}
\end{figure}

Figure~\ref{fig:appendix-pred-vs-truth} is particularly useful for diagnosing model error. Instead of only reporting a scalar metric, it overlays the model-predicted counterfactual with the environment-truth counterfactual generated by MuJoCo. The important question is not just whether the final state is correct, but when and how the predicted trajectory starts to deviate. In Push, small geometric deviations can accumulate near the success boundary and flip the event label even when the overall path remains close. In Door, the comparison is more sensitive to whether the model preserves the correct contact-side and opening branch. The figure therefore explains why event metrics and geometric metrics can separate differently across the two tasks.

\begin{figure}[H]
\centering
\includegraphics[width=\linewidth]{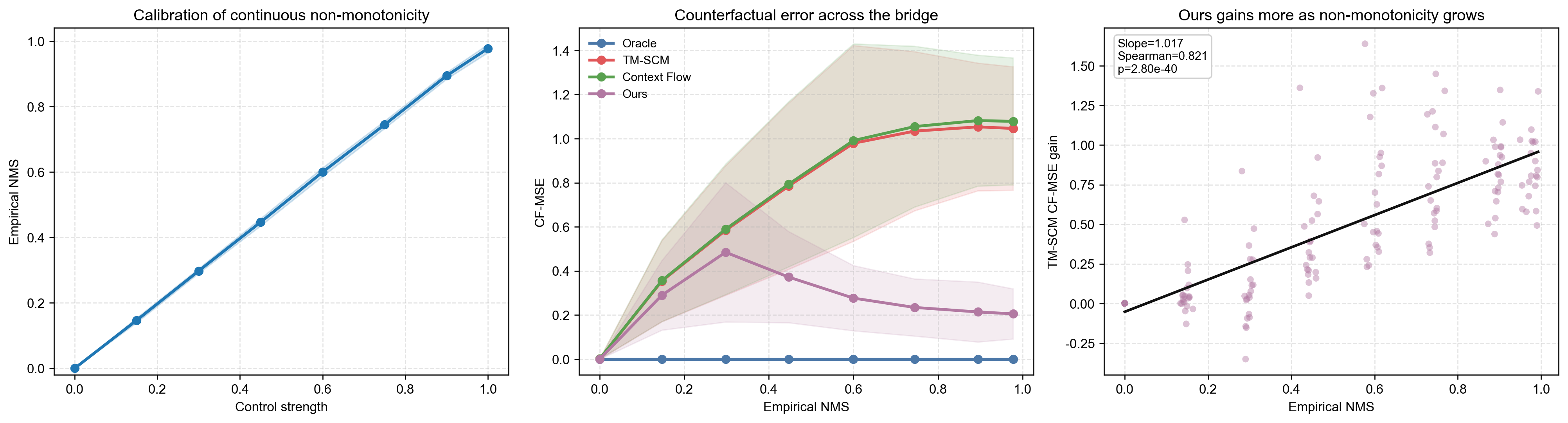}
\caption{Continuous synthetic bridge from globally monotone to strongly non-monotone mechanisms.}
\label{fig:appendix-synth-bridge}
\end{figure}

Figure~\ref{fig:appendix-synth-bridge} turns the synthetic claim into a continuous spectrum rather than a three-point family comparison. The left panel verifies that the control parameter is well calibrated: the realized synthetic non-monotonicity score increases almost linearly from $0$ to about $0.98$. The middle panel then shows the counterfactual consequence of that sweep. Near the monotone end, TM-SCM, contextual flow, and our model are all competitive, which is exactly the regime where a stronger inverse-transport bias should not be expected to create a large advantage. As the mechanism becomes more non-monotone, however, the monotone prior becomes increasingly misspecified: TM-SCM and contextual flow errors rise sharply, while the error of our model eventually decreases after the moderate regime and remains much lower in the strongly non-monotone region. The right panel makes the boundary argument explicit at the per-run level. The gain of our model over TM-SCM grows with the calibrated non-monotonicity score, with fitted slope $1.017$ and Spearman $\rho=0.821$ ($p=2.8\times 10^{-40}$). This is the synthetic analogue of the Push-versus-Door story in the main text: stronger branch-sensitive non-monotonicity makes inverse-transport alignment more valuable.

\begin{figure}[H]
\centering
\includegraphics[width=0.62\linewidth]{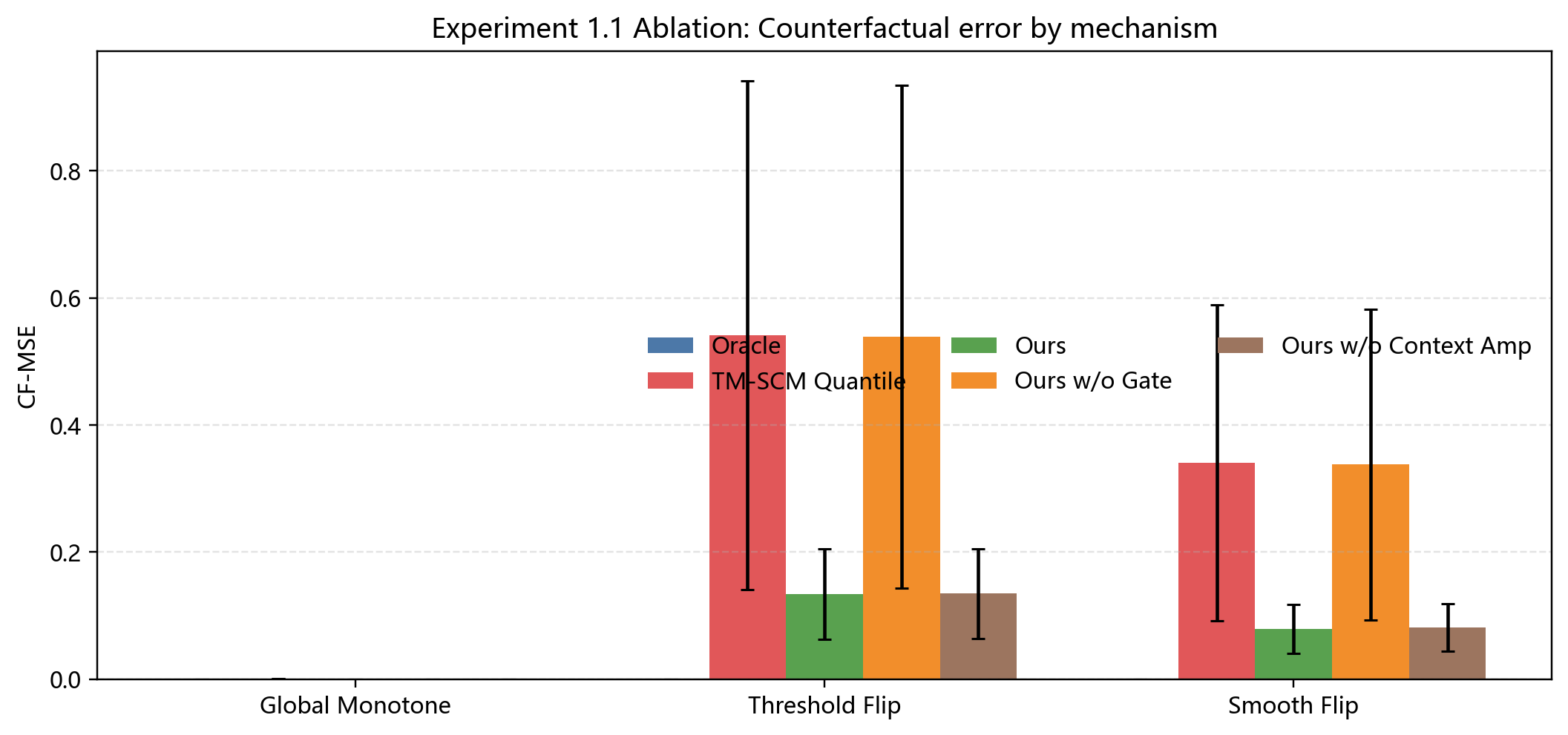}
\caption{Synthetic ablations for Experiment 1.1.}
\label{fig:appendix-synth-ablation}
\end{figure}

Figure~\ref{fig:appendix-synth-ablation} complements the synthetic main table by isolating which parts of the model matter most once the benchmark mechanisms become non-monotone. The main pattern is that removing inversion- or transport-related structure causes the largest counterfactual degradation on the threshold-flip and smooth-flip families, whereas globally monotone mechanisms are much less sensitive. This matters because the synthetic benchmark is the cleanest setting in which the structural claim can be separated from embodied-control noise. The figure therefore supports the interpretation that the proposed inductive bias is not universally beneficial in an unconstrained sense; rather, it is specifically beneficial when the data-generating mechanism actually contains the kind of branch-sensitive non-monotonicity targeted by the theory.

\section{Reproducibility Statement}

All reported results are produced by repository scripts, with raw \texttt{json}, \texttt{md}, and figure artifacts stored under \texttt{code/outputs/}. The discrete synthetic sweep covers $108$ configurations, the continuous synthetic bridge covers $160$ calibrated runs, the formal MuJoCo Push replication is generated by \texttt{run\_multiseed\_4\_3\_push\_mujoco\_balanced.py}, the Door robustness study by \texttt{run\_robustness\_4\_3\_door\_mujoco.py}, the paired physical statistics by \texttt{run\_physical\_significance.py}, and the Door three-seed head-to-head by \texttt{run\_multiseed\_4\_3\_door\_mujoco\_head2head.py}. The repository also includes a numeric consistency audit, \texttt{audit\_physical\_claims.py}, which checks that the physical numbers quoted in the paper match the generated artifacts. Dataset directories contain \texttt{metadata.json} files with seeds, sizes, environment parameters, and sampler settings; full proofs and detailed result tables are deferred to the supplement.

\section{Limitations}

Our guarantees apply to shared-order triangular recursive SCMs and do not directly cover strongly cyclic systems. They also rely on mechanism-wise invertibility, so severe dissipation or many-to-one compression can break factual abduction itself. Empirically, CausalInverter is only a theory-inspired approximation, and the embodied evaluation is deliberately state-based and low-trajectory: we now compare against compact feedforward, autoregressive, recurrent, Transformer, and flow baselines under the same limited supervision, but not yet against image-conditioned world models or deep latent causal baselines trained directly on the physical tasks. Thus the current empirical claim is narrower than a general embodied-world-model claim. Finally, the strongest physical evidence comes from Door; Push remains a formally re-checked boundary case where stronger low-data predictors stay competitive or better.

\end{document}